\documentclass{article}

\usepackage{arxiv}
\usepackage[utf8]{inputenc} % allow utf-8 input
\usepackage[T1]{fontenc}    % use 8-bit T1 fonts
\usepackage{hyperref}       % hyperlinks
\usepackage{url,comment}            % simple URL typesetting
\usepackage{booktabs}       % professional-quality tables
\usepackage{amsfonts,amsmath,amsthm,amssymb}       % blackboard math symbols
\usepackage{mathrsfs} % for \mathscr
\usepackage{nicefrac}       % compact symbols for 1/2, etc.
\usepackage{microtype}      % microtypography
\usepackage{lipsum}		% Can be removed after putting your text content
\usepackage{graphicx}
\usepackage{natbib}
\usepackage{doi}
\usepackage[skins]{tcolorbox}
\usepackage{tikz}
\usetikzlibrary{arrows.meta, positioning, shapes.geometric, decorations.pathreplacing, calc}
\usepackage{adjustbox,subcaption}
\usepackage{mathtools}

\newtheorem{definition}{Definition}

\newtheorem{theorem}{Theorem}
\newtheorem{lemma}{Lemma}
\newtheorem{proposition}{Proposition}
\newtheorem{corollary}{Corollary}
\newtheorem{principle}{Principle}

\newtheorem{example}{Example}

\title{Cycle is All You Need: More Is Different}

\author{ \href{https://orcid.org/0000-0003-2067-2763}{\includegraphics[scale=0.06]{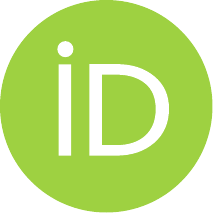}\hspace{1mm}Xin Li}\thanks{This work was partially supported by NSF IIS-2401748 and BCS-2401398. The author has used ChatGPT models to assist the development of theoretical ideas and visual illustrations presented in this paper.} \\
	Department of Computer Science\\
	University at Albany\\
	Albany, NY 12222 \\
	\texttt{xli48@albany.edu} 
}

\renewcommand{\shorttitle}{\textit{arXiv} Template}

\hypersetup{
pdftitle={A template for the arxiv style},
pdfsubject={q-bio.NC, q-bio.QM},
pdfauthor={David S.~Hippocampus, Elias D.~Striatum},
pdfkeywords={First keyword, Second keyword, More},
}

\begin{document}
\maketitle

\begin{abstract}
We propose an information topological framework for intelligence in which \emph{cycle closure} is the fundamental mechanism of memory. 
Building on the first principle, we argue that memory is best understood not as a static store of representations, but as the ability to \emph{re-enter and traverse latent cycles} in neural state space.  
We identify these invariant cycles as the natural carriers of meaning across scales: they act as \emph{alignment checkpoints} between context ($\Psi$) and content ($\Phi$), filtering out order-specific noise, enforcing closure, and preserving only what remains consistent across variations. 
A key principle underlying this framework is the \emph{dot-cycle dichotomy}: trivial cycles collapse to dots ($H_0$), serving as transient contextual scaffolds ($\Psi$), while nontrivial cycles ($H_1$ and higher) encode low-entropy content invariants ($\Phi$) that persist as memory. 
This dichotomy clarifies how cognition achieves both adaptability and stability: dots support exploration, while cycles carry persistent knowledge across contexts. 
Biologically, this dichotomy is implemented by \emph{polychronous neural groups} (PNGs), which realize closed $1$-cycles through delay-locked spiking trajectories stabilized by spike-timing dependent plasticity (STDP). 
Oscillatory phase coding embeds PNGs within global temporal scaffolds, with slow rhythms (e.g., theta) segmenting experience into macro-cycles and fast rhythms (e.g., gamma) encoding discrete content packets. 
Coincidence detection enforces boundary cancellation ($\partial^2=0$), ensuring that only reproducible, order-invariant cycles survive. 
Through cross-frequency nesting and replay, these micro-cycles are recursively composed into a hierarchy of cycles, yielding a multiscale architecture of memory persistence and generalization ("more is different"). 
The perception-action cycle further generalizes this principle, introducing \emph{high-order invariance}: not only does order not matter within perception or action, but even the sequencing of sense-act alternations is irrelevant so long as the cycle closes on the intended goal. 
Evolutionarily, such mechanisms extend the homing cycles of spatial navigation, making navigation the ancestral substrate of structured memory. 
Finally, the sheaf-cosheaf duality provides a unifying mathematical 
language: contextual scaffolds ($\Psi$) glue local features into global 
sections (sheaf), while content cycles ($\Phi$) extend into global plans 
(cosheaf). Closure in this dual structure aligns top-down predictive 
potentials with bottom-up experiential cycles. By treating cycles as the 
algebraic footprint of broken symmetry, and by grounding their biological 
realization in PNG hierarchies, our framework unifies perception, action, 
and memory as instances of latent navigation. Beyond cognition, this same 
closure principle illuminates the nature of consciousness: awareness arises 
when local fragments of perception and global plans align into persistent 
cycles that integrate (unity) yet remain differentiated (richness) across 
contexts. Conscious experience is thus the phenomenological correlate of 
cycle persistence, the survival of high-order invariants across time and 
modality. We conclude that \emph{cycle is all you need}: memory and 
consciousness emerge from persistent topological invariants, enabling 
intelligent systems to generalize in non-ergodic environments while 
maintaining long-term coherence with minimal energetic cost.

\end{abstract}

% keywords can be removed

\section{Introduction}
\label{sec:1}

%\subsection*{The Four No's from Wheeler's \emph{It from Bit}}

John Archibald Wheeler’s celebrated dictum, \emph{“It from Bit”} \cite{wheeler2018information}, asserts that physical reality is fundamentally informational: what we perceive as ``things'' 
are distilled from acts of information exchange. 
In this seminal essay, 
Wheeler articulated four guiding negations, the ``Four No's'', as 
constraints on any ultimate account of physical reality:

\begin{enumerate}
  \item \textbf{No tower of turtles.}  
  There is no infinite regress of structures, each supported by another 
  deeper structure. Instead, Wheeler proposes closure in the form of cycles:  
  physics $\to$ observer-participancy $\to$ information $\to$ physics.

  \item \textbf{No laws.}  
  The laws of physics are not timeless Platonic entities; they must have 
  come into being. Wheeler emphasizes ``law without law,'' grounded in 
  principles like the topological identity $\partial^2 = 0$, rather than 
  pre-existing mechanical blueprints.

  \item \textbf{No continuum.}  
  There is no fundamental continuum of space, time, or even probability.  
  The continuum is a convenient mathematical idealization, while physics 
  at root is discrete and information-theoretic, founded on irreversible 
  quantum events (yes/no registrations).

  \item \textbf{No space, no time.}  
  Space and time are not absolute backdrops; they are emergent orders of 
  things, tools of description, not primordial entities. At Planck scales, 
  even connectivity and temporal order lose meaning, requiring a reformulation 
  of existence in terms of bits rather than spacetime.
\end{enumerate}

What is the implication of It-from-Bit for intelligence? We propose that intelligence itself can be understood as a process of extracting 
persistent structure from streams of information, where what survives is not 
the fleeting bit, but the invariant \emph{cycle} \cite{davatolhagh2024bit}. 

\begin{principle}[First Principle of Intelligence]
Intelligence is the capacity to stabilize invariants by cycle closure. 
At its core, cognition operates by minimizing joint context-content 
uncertainty $H(\Psi,\Phi)$, eliminating dangling boundaries and promoting 
them into closed cycles. These cycles constitute the fundamental units of 
meaning, memory, and prediction. 
\end{principle}

Our guiding claim is that \emph{cycle is all you need}: the organization of 
cognition, memory, and abstract thoughts in neural systems follows from the universal role of cycles 
as the algebraic residue of broken symmetry and the topological skeleton 
of information flow \cite{li2025information}. In the spirit of Wheeler, we propose the following four No's for cognition.

\subsection*{The Four-No's for Cognition}

To articulate the informational constraints on cognition, we advance four 
foundational ``No’s'' that serve as boundary conditions on any viable theory:

\begin{enumerate}
    \item \textbf{No isolated information.} Bits are never standalone: they 
    acquire meaning only through relations that close into cycles. 
    Information without recurrence dissipates as noise.
    \item \textbf{No privileged order.} The cognitive system must be robust 
    to permutations of local steps (e.g., \ eye movements, motor primitives). 
    What matters is closure into a cycle, not the linear order of micro-events.
    \item \textbf{No static storage.} Memory is not a passive container of 
    representations, but a dynamic ability to re-enter latent cycles in state space.
    \item \textbf{No prediction without invariance.} Forecasting future states 
    requires reducing entropy by filtering order-dependent variations; only 
    invariant cycles can stabilize the predictive substrate.
\end{enumerate}

\noindent\textbf{From constraints to clues.} These four principles define cognition as a non-ergodic information process \cite{li2025Beyond}: 
rather than averaging over all possible trajectories, the mind concentrates 
its dynamics onto recurrent, invariant cycles that persist across perturbations. 
Taken together, the Four No's funnel cognition toward recurrent organization: items must close into cycles (no isolated information), be insensitive to micro-order (no privileged order), support re-entry (no static storage), and stabilize invariants for prediction (no prediction without invariance).
The lightest formalism that enforces all four at once is the chain complex with boundary operator $\partial$ \cite{hatcher2002algebraic}: its nilpotency, $\partial^2=0$, cancels stray endpoints so that only closed traversals remain. 
This is the key new insight underlying the dot-cycle dichotomy, as shown in Fig.~\ref{fig:dot-cycle-dichotomy}, and it sets up our first clue.

% ---------- TikZ: ∂^2=0 enforces Dot (H0) vs cycle (H1) ----------
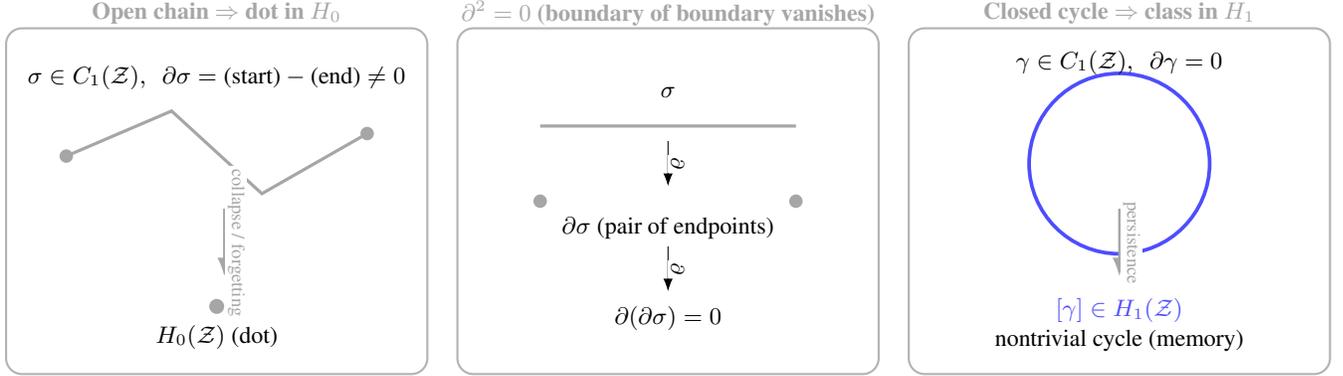
\begin{figure}[t]
\centering
\begin{tikzpicture}[
  >=Latex,
  scale=1.0,
  every node/.style={font=\small},
  panel/.style={rounded corners=6pt, draw=gray!60, line width=0.8pt, minimum width=5.6cm, minimum height=4.6cm},
  title/.style={font=\small\bfseries, text=gray!70},
  chainline/.style={line width=1.2pt, gray!70},
  cycleline/.style={line width=1.4pt, blue!70},
  endpoint/.style={circle, fill=gray!70, inner sep=1.8pt},
  arrowlab/.style={midway, above, sloped, fill=white, inner sep=1pt}
]

% Panels
\node[panel] (LeftBox)  at (-6.0,0) {};
\node[panel] (MidBox)   at (  0.0,0) {};
\node[panel] (RightBox) at (  6.0,0) {};

\node[title] at ([yshift=2.5cm]LeftBox)  {Open chain $\Rightarrow$ dot in $H_0$};
\node[title] at ([yshift=2.5cm]MidBox)   {$\partial^2=0$ (boundary of boundary vanishes)};
\node[title] at ([yshift=2.5cm]RightBox) {Closed cycle $\Rightarrow$ class in $H_1$};

% ----- LEFT: Open chain with nonzero boundary -----
% Polyline (open chain)
\draw[chainline] ($(LeftBox)+(-2.0,0.6)$) -- ($(LeftBox)+(-0.6,1.2)$) --
                 ($(LeftBox)+( 0.6,0.1)$) -- ($(LeftBox)+( 2.0,0.9)$);

% Endpoints (boundary points)
\node[endpoint] (LStart) at ($(LeftBox)+(-2.0,0.6)$) {};
\node[endpoint] (LEnd)   at ($(LeftBox)+( 2.0,0.9)$) {};

% Labels
\node at ($(LeftBox)+(0,1.65)$) {$\sigma \in C_1(\mathcal{Z}),\ \ \partial\sigma = \text{(start)} - \text{(end)} \neq 0$};

% Arrow: collapse to H0 dot
\draw[->, gray!70, line width=0.9pt] ($(LeftBox)+(0.1,-0.1)$) -- ($(LeftBox)+(0.1,-1.0)$)
  node[arrowlab] {\scriptsize collapse / forgetting};
\node[circle, fill=gray!70, inner sep=2pt] at ($(LeftBox)+(0,-1.4)$) {};
\node at ($(LeftBox)+(0,-1.8)$) {$H_0(\mathcal{Z})$ (dot)};

% ----- MIDDLE: Boundary of boundary vanishes -----
% Show σ, then ∂σ (two dots), then ∂(∂σ)=0
\node at ($(MidBox)+(0,1.45)$) {$\sigma$};
\draw[chainline] ($(MidBox)+(-1.7,1.0)$) -- ($(MidBox)+(1.7,1.0)$);

% First boundary
\draw[->] ($(MidBox)+(0,0.8)$) -- ($(MidBox)+(0,0.2)$) node[arrowlab] {\scriptsize $\partial$};
\node[endpoint] at ($(MidBox)+(-1.7,0.0)$) {};
\node[endpoint] at ($(MidBox)+( 1.7,0.0)$) {};
\node at ($(MidBox)+(0,-0.35)$) {$\partial\sigma$ (pair of endpoints)};

% Second boundary (vanishes)
\draw[->] ($(MidBox)+(0,-0.6)$) -- ($(MidBox)+(0,-1.2)$) node[arrowlab] {\scriptsize $\partial$};
\node at ($(MidBox)+(0,-1.55)$) {$\partial(\partial\sigma)=0$};

% ----- RIGHT: Closed cycle with zero boundary -----
% cycle (circle or rounded cycle)
\draw[cycleline] ($(RightBox)+(0,0.5)$) circle (1.2);
\node at ($(RightBox)+(0,1.85)$) {$\gamma \in C_1(\mathcal{Z}),\ \ \partial\gamma = 0$};

% Arrow: class in H1
\draw[->, gray!70, line width=0.9pt] ($(RightBox)+(0,-0.1)$) -- ($(RightBox)+(0,-1.0)$)
  node[arrowlab] {\scriptsize persistence};
\node[blue!70] at ($(RightBox)+(0,-1.45)$) {$[\gamma] \in H_1(\mathcal{Z})$};
\node at ($(RightBox)+(0,-1.85)$) {nontrivial cycle (memory)};

\end{tikzpicture}
\caption{\textbf{$\partial^2=0$ enforces the dot-cycle dichotomy.} 
\emph{Left:} An open chain $\sigma$ has a nonzero boundary $\partial\sigma$ and
collapses to a dot (class in $H_0$), carrying no relational content. 
\emph{Middle:} The boundary operator squares to zero: $\partial(\partial\sigma)=0$. 
\emph{Right:} A closed chain $\gamma$ with $\partial\gamma=0$ persists as a homology
class $[\gamma]\in H_1$, i.e., a cycle that encodes order-invariant structure.}
\label{fig:dot-cycle-dichotomy}
\end{figure}

\begin{corollary}[The Boundary of a Boundary Vanishes]
Under the First Principle, intelligence is realized 
through cycle closure. This closure is only possible because the boundary 
operator $\partial$ satisfies the fundamental identity
$\partial^2 = 0$.
That is, the boundary of a boundary vanishes. Cognitively, this law ensures that when cognition promotes boundaries into cycles, no further inconsistencies remain at the next level: every open edge is paired, every fragment canceled. This guarantees the 
existence of stable invariants (cycles), which are the carriers of meaning, 
memory, and communication. 
Therefore, $\partial^2=0$ constitutes the \emph{First Clue} of 
intelligence: coherence arises because boundaries consistently vanish 
when lifted, enabling cycles to persist.
\end{corollary}

\subsection*{The First Clue: The Boundary of a Boundary Vanishes}

The topological underpinning of our framework begins with a simple but profound 
identity: 
$\partial^2 = 0$.
This statement, that the boundary of a boundary vanishes, is the first 
clue for how memory must be organized \cite{squire1993structure}. In homology, it ensures that open 
chains (unclosed fragments of experience) cancel at their boundaries, leaving 
behind only closed cycles as persistent invariants \cite{hatcher2002algebraic}. In cognition, this translates into the principle that incomplete or inconsistent relations are filtered out, 
while closed cycles of relations endure as stable memory traces \cite{baddeley1997human}. 

\noindent\textbf{Dot–cycle Dichotomy.}
At the chain level, a “dot” (0–simplex) records isolated content, whereas a “cycle” (1–cycle) captures a closed relation in which endpoints cancel.
The rule $\partial^2=0$ formalizes this passage: boundaries of fragments do not compose, but pairwise cancellation at endpoints yields a cycle that survives in homology.
Cognitively, this is the move from token to trace \cite{li2025Delta}: contents $\Phi$ are registered as dots, yet only when linked by contextual relations $\Psi$ into a closed cycle do they consolidate as durable memory.
This bridge motivates the claim that the primitive of memory is not the item but the cycle, setting up the biological mechanisms below.
Coincidence detection in neural circuits provides a biological realization 
of this algebraic rule: inputs that do not align temporally cancel, whereas 
synchronous inputs reinforce a closed trajectory \cite{konig1996integrator}. Similarly, oscillatory phase coding cyclically parameterizes time, ensuring that every cognitive boundary 
is embedded in a cycle \cite{buzsaki2006rhythms}. 

The vanishing of boundaries thus guarantees that what remains in memory is not arbitrary fragments but coherent cycles: the minimal 
invariants that bind context and content into an intelligible whole. From a computational standpoint, this marks a profound departure from the Turing paradigm \cite{turing1936computable}. Traditional machines rely on symbolic tokens and sequential operations, where meaning is assigned externally to states of a register or tape. In contrast, a cycle-based architecture derives meaning intrinsically from topological closure: invariants are not “written” into memory but emerge from the very dynamics of neural interaction \cite{gerstner2014neuronal}. This dot–cycle dichotomy, where trivial cycles collapse and only nontrivial cycles persist, provides a natural mechanism for error correction, generalization, and energy efficiency without requiring exhaustive symbolic manipulation or gradient descent over high-dimensional parameter spaces.
Such a principle opens a non-Turing path for AI: instead of operating on static representations, machines could operate on homological carriers of information that stabilize through $\partial^2=0$ \cite{wheeler2018information}. Closure rather than linear update rules would regulate information flow, and persistence of memory would follow from structural invariants rather than explicit storage. In this sense, cycle-based computation is not just a metaphor but a new substrate for building cognitive machines, one in which intelligence arises from organizing dynamics into cycles rather than executing instructions in sequence.

\subsection*{From It-from-Bit to Cognition-from-Cycle}

By situating cognition within Wheeler’s informational ontology, we suggest a 
refinement: \emph{It is from cycle}. Cycles, not bits, are the natural atoms 
of cognition \cite{li2025Beyond}. They reconcile the demands of memory persistence, order invariance, 
and predictive efficiency by embodying what remains invariant when symmetry is broken and boundaries cancel. 
This refinement can be traced directly to the first topological clue, 
$\partial^2=0$, the principle that the boundary of a boundary vanishes. 
In cognitive dynamics, this identity guarantees that open or inconsistent 
relations cancel at their endpoints, leaving only closed cycles as stable 
invariants. Two fundamental neural mechanisms instantiate this principle 
in biological hardware. 
First, \emph{oscillatory phase coding} provides a cyclic parameterization 
of time \cite{russo2024integration}: neural events are not encoded as isolated instants but as phases 
on an ongoing oscillation (theta, gamma, etc.). Every oscillation is a cycle 
on $S^1$, so temporal flow is continuously folded back into a cycle. This 
cyclic scaffolding ensures that beginnings and endings align, enforcing the 
closure condition required for persistent memory. 
Second, \emph{coincidence detection} enforces boundary cancellation. Neurons that fire only when multiple inputs arrive within a narrow temporal window 
implement the algebraic rule $\partial^2=0$ at the level of spikes. Inputs 
that fail to align are suppressed, while those that synchronize reinforce 
a closed trajectory of activity \cite{matell2004cortico}. In this way, biological circuitry eliminates 
dangling boundaries and retains only cycles that embody stable relations. 
Together, phase coding and coincidence detection transform Wheeler’s 
\emph{It-from-Bit} into what we call \emph{Cognition-from-Cycle}: oscillations 
supply the circular scaffold for cognition, and coincidence detection prunes 
inconsistencies so that only closed cycles persist. 

This perspective on cognition has a striking parallel in Grothendieck’s philosophy of mathematics \cite{grothendieck1957quelques}. Grothendieck argued that mathematics should not begin with isolated problems or ad hoc tricks, but with the most general structures in which those problems naturally reside. For him, the power of mathematics lay in constructing universal frameworks, schemes, sheaves, cohomology \cite{hatcher2002algebraic}, so that particular cases appear merely as shadows of deeper structures. The same principle underlies the dot–cycle dichotomy in cognition: dots represent atomic, specific instances, while cycles represent the stable structures or cycles that emerge when those dots are consistently related. Just as Grothendieck’s use of sheaf theory and cohomology guarantees that local data can be glued into global invariants, cognition uses closure and cycle-formation to stabilize raw experiences into coherent memory and meaning. The zigzag lemma provides the algebraic mechanism for this transformation \cite{hatcher2002algebraic}: boundaries (dots that cannot close locally) are promoted into cycles (cycles that persist globally). Based on the first principle, cognition mirrors Grothendieck’s structural method: coherence is achieved not by the accumulation of dots, but by ensuring that every dot finds its closure in a cycle, unifying specifics under universal structure.

The rest of this paper 
develops this principle formally, showing how cycle formation unifies 
perception, action, and memory into a single topological framework. We first discuss how broken symmetry leads to invariant cycles in Sec. \ref{sec:2}. In Sec. \ref{sec:3}, we lay out the biological foundation of hierarchical cycle formation via polychronous neural groups. Based on topological closure, we present a unified cognition-from-cycle framework from the perspective of bootstrapping in Sec. \ref{sec:4}. But ``more is different'', the hierarchy of cycle formation has enabled the generation of more sophisticated memory systems and intelligent designs. In Sec. \ref{sec:5}, we study high-order invariance and interpret conscious experience as the phenomenological correlate of 
cycle persistence, the survival of high-order invariants across time and 
modality. We discuss the implications of Cognition-from-Cycle for AGI and cognitive computation in Sec. \ref{sec:6}.

\section{Motivation: from Invariant Cycles to Dynamic Alignment}
\label{sec:2}

\subsection{Topological View: Invariant Cycles as the Non-Ergodic Counterpart to Measure Preservation}

Classical ergodic theory is built on the notion of a measure-preserving transformation \cite{walters2000introduction}.
A dynamical system $(X,\mathcal{B},\mu,T)$ consists of a probability space
$(X,\mathcal{B},\mu)$ and a measurable transformation $T:X\to X$ satisfying
$\mu(T^{-1}A) = \mu(A), \quad \forall A \in \mathcal{B}$.
This measure invariance guarantees that long-term time averages along almost every trajectory
coincide with ensemble averages with respect to $\mu$. In this setting, entropy
(e.g., Kolmogorov-Sinai entropy \cite{cornfeld2012ergodic}) quantifies the unpredictability of the evolution
under the assumption of ergodicity.
Intelligent systems, however, are fundamentally non-ergodic \cite{li2025Beyond}: they retain memory,
exhibit path dependence, and actively reduce uncertainty. In such systems,
the measure $\mu$ is not preserved, but typically \emph{concentrated}
onto lower-dimensional recurrent structures through learning and adaptation \cite{spisak2025self}.
This concentration corresponds to entropy minimization rather than entropy conservation.

We propose that the appropriate generalization of ``measure-preservation'' in the
non-ergodic setting is \emph{cycle-preservation}. That is, while probability measures
are not conserved globally, the system preserves \emph{topological invariants}
encoded in cycles that represent memory traces and recurrent behavioral motifs.
Formally, let $(X,T)$ be a discrete-time dynamical system on a topological state space $X$. A $k$-cycle is a chain $\gamma \in Z_k(X)$ satisfying $\partial \gamma = 0$.
Under the induced map $T_\ast$ on chains, invariance of $\gamma$ requires that
$T_\ast \gamma - \gamma = \partial \beta$
for some $(k+1)$-chain $\beta$. Equivalently,
$[T_\ast \gamma] = [\gamma]~ \text{in}~ H_k(X)$, where $H_k(X)$ denotes the 
k-th homology group of the topological space $X$,
so that $\gamma$ is invariant up to homology class. In this way, although
trajectories deform under dynamics, the \emph{memory} encoded by the homology class
persists.

\begin{table}[h!]
\centering
\caption{Comparison between ergodic and non-ergodic frameworks.}
\begin{tabular}{|p{3.5cm}|p{5.5cm}|p{5.5cm}|}
\hline
\textbf{Aspect} & \textbf{Ergodic Theory} & \textbf{Non-Ergodic (Adaptive Systems)} \\
\hline
Invariant structure & Measure-preserving transformations ($\mu(T^{-1}A)=\mu(A)$) & Measure concentration on recurrent, lower-dimensional cycles or attractors \\
\hline
Entropy dynamics & Entropy is conserved or grows under mixing (Kolmogorov--Sinai entropy) & Entropy minimized via concentration on structured states (Friston 2012) \\
\hline
Typical trajectories & Explore the entire space given enough time (ergodicity) & Converge onto invariant cycles, attractors, or heteroclinic orbits \\
\hline
Statistical averaging & Time averages = ensemble averages & Time averages biased by concentration; ensemble averages may not exist \\
\hline
Memory mechanism & No persistent memory beyond statistical mixing & Memory encoded as stable homology classes $[\gamma]\in H_k$ (cycles) \\
\hline
Prediction & Statistical prediction based on invariant measure & Structural prediction based on persistence of invariant cycles \\
\hline
\end{tabular}
\label{tab:duality}
\end{table}

This shift in perspective reframes the role of entropy reduction. 
In ergodic systems, entropy is managed by distributing trajectories 
uniformly across the entire state space $X$, ensuring statistical 
equivalence of time and ensemble averages. By contrast, in non-ergodic, 
adaptive systems, entropy reduction is achieved through 
\emph{measure concentration}: rather than exploring all of $X$, 
trajectories are funneled toward lower-dimensional recurrent sets. 
These recurrent sets correspond to \emph{persistent cycles} that remain 
stable under perturbations and across variations in initial conditions. 
In this sense, cycles act as the carriers of invariant information, 
preserving structural regularities across history-dependent dynamics 
and filtering out order-specific noise. The outcome is that 
intelligence emerges not from uniform exploration, but from the ability 
to stabilize information flow through the persistence of these invariant 
structures, as shown in Table \ref{tab:duality}. Formally, we have

\begin{proposition}[Non-Ergodic Invariance Principle]
\label{prop:nonergodic-invariance}
Let $(X,T)$ be a dynamical system on a topological state space $X$.
Then the natural counterpart of measure-preservation in ergodic theory is
\emph{cycle-preservation}:
$T_\ast : H_k(X) \to H_k(X), [\gamma] \mapsto [\gamma]$.
That is, an intelligent system preserves homology classes of cycles
even while its measure evolves non-uniformly. These invariant cycles formalize
memory persistence as the structural backbone of cognition.
\end{proposition}

When a non-ergodic system with many symmetric possibilities is forced to choose one outcome, 
symmetry is broken \cite{li2025Broken}. In neural and cognitive dynamics, this choice does not erase the 
unselected alternatives; instead, it organizes them into a closed cycle of relations:
the chosen state, its competitors, and the transitions among them. 
In other words, the brain does not simply ``pick a winner'' among symmetric options. 
It establishes a cycle that records the selection, keeps the alternatives accessible 
for recall or switching, and stabilizes the outcome through recurrent interaction \cite{li2025Beyond}. 
Broken symmetry, therefore, inevitably produces cycle formation, since the invariant 
residue of selection is a cycle connecting choice, memory, and potential revision.

This perspective reframes the role of entropy in prediction. 
Proposition~\ref{prop:nonergodic-invariance} establishes that non-ergodic systems 
preserve homology classes of cycles as their structural invariants. 
From an information-theoretic viewpoint, symmetry corresponds to maximal 
uncertainty: if all outcomes are equivalent under a symmetry group $G$, the induced 
distribution is uniform (entropy is maximized). Symmetry breaking reduces this 
uncertainty by eliminating redundant possibilities, thereby lowering entropy and 
concentrating probability mass around residual invariant cycles. 
In high dimensions, this process can be understood through the 
\emph{theory of measure concentration} \cite{ledoux2001concentration}: instead of spreading trajectories uniformly 
(as in ergodic systems), the dynamics of learning and memory focus trajectories 
around persistent cycles. To make this precise, we introduce the notion of \emph{residual invariants} \cite{beekman2019introduction}: the structural survivors of symmetry breaking concentrate probability 
mass onto persistent cycles and formalize what remains stable under the 
reduced symmetry subgroup.
 %We use a series of lemmas to show how these cycles act as attractors of probability flow, providing reliable, order-invariant substrates for prediction across time.

\begin{definition}[Residual Invariants under Symmetry Breaking]
\label{def:residual-invariants}
Let a system evolve on a state space $\mathcal{Z}$ with symmetry group $G$. 
Suppose a perturbation $\varepsilon$ breaks $G$-equivariance, reducing the 
symmetry to a subgroup $H \subset G$ and forcing selection of a representative 
state $\Phi_\varepsilon \in \mathcal{Z}$. The \emph{residual invariants} are 
those structures that remain preserved under $H$ despite the breaking of $G$. 
Formally, they are equivalence classes of cycles 
$[\gamma] \in H_k(\mathcal{Z})$
that are stable under $H$-action and persist under perturbations of $\varepsilon$.  
\end{definition}

Intuitively, residual invariants encode what remains stable after a decision 
or perturbation: in physics, they correspond to conserved quantities or 
Goldstone modes \cite{beekman2019introduction}; in topology, to persistent homology classes \cite{edelsbrunner2008persistent}; and in cognition, 
to cycles that bind chosen outcomes with unchosen alternatives, enabling recall, 
revision, and reuse \cite{chen2023our}. This intuition can be formalized by showing that residual invariants 
emerging from symmetry breaking necessarily take the form of closed cycles, 
which persist as homology classes and provide the structural foundation of memory.

\begin{lemma}[Symmetry Breaking Generates Invariant Cycles]
\label{lem:symm-invariant}
Let a system evolve on a state space $\mathcal{Z}$ with symmetry group $G$.
Suppose a perturbation $\varepsilon$ breaks $G$-equivariance by forcing the selection of a representative state $\Phi_\varepsilon$. Then:
1) The broken symmetry induces residual structures (orbits)
invariant under residual transformations $H\subset G$.
2) These residual invariants manifest as closed cycles
$\gamma\subset \mathcal{Z}$ stabilized by feedback (i.e.\ $\partial \gamma=0$).
3) $\gamma$ defines a homology class
$[\gamma]\in H_k(\mathcal{Z})$ that is stable under perturbations of $\varepsilon$,
formalizing memory persistence.
\end{lemma}

The lemma establishes that symmetry breaking inevitably leaves behind 
residual invariants in the form of cycles, which act as stable memory traces 
of past selections. To fully understand their cognitive function, one must 
ask: What advantage does the system gain from organizing dynamics into such 
closed cycles? The key lies in the fact that cycles identify equivalence classes 
of trajectories, collapsing many superficially different paths into the same 
topological invariant \cite{hatcher2002algebraic}. In other words, once dynamics are organized into 
homology classes, prediction and memory no longer depend on the precise 
\emph{order} of steps, but only on the closure of the cycle \cite{li2025Beyond}. This observation 
leads directly to the following corollary: cycles serve as the structural 
basis of \emph{order invariance}, ensuring robustness in navigation, perception, action, 
and more abstract cognitive computations.

% ---------- TikZ: Trivial vs Nontrivial 1-cycles ----------
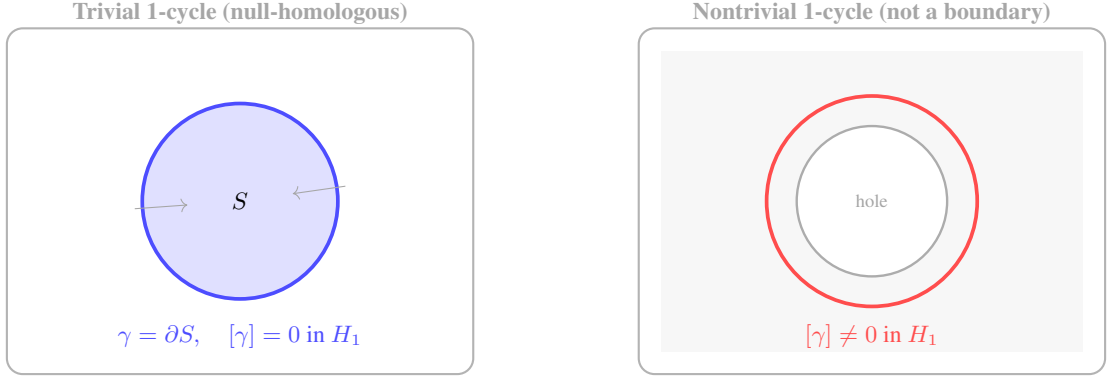
\begin{figure}[t]
\centering
\begin{tikzpicture}[
  scale=1.0,
  every node/.style={font=\small},
  panel/.style={rounded corners=6pt, draw=gray!60, line width=0.8pt, minimum width=6.2cm, minimum height=4.6cm},
  cycletriv/.style={line width=1.4pt, blue!70},
  cyclenontriv/.style={line width=1.4pt, red!70},
  fillface/.style={fill=blue!12, draw=none},
  holedisc/.style={fill=white, draw=gray!65, line width=0.9pt}
]

% ---------------- Left panel: trivial 1-cycle (boundary of a 2-chain) ----------------
\begin{scope}[shift={(-4.2,0)}]
  \node[panel] (LeftBox) at (0,0) {};
  \node[gray!70] at (0,2.5) {\bfseries Trivial 1-cycle (null-homologous)};

  % A filled face S (disk)
  \fill[fillface] (0,0) circle (1.3);

  % Its boundary γ = ∂S
  \draw[cycletriv] (0,0) circle (1.3);

  % Labels
  \node at (0,0) {$S$};
  \node[blue!70] at (0,-1.8) {$\gamma=\partial S,\quad [\gamma]=0 \ \text{in}\ H_1$};

  % (Optional) A small inward arrow to suggest “fillable”
  \draw[->, gray!70] (1.4,0.2) -- (0.7,0.1);
  \draw[->, gray!70] (-1.4,-0.1) -- (-0.7,-0.05);
\end{scope}

% --------------- Right panel: nontrivial 1-cycle (cycles around a hole) ---------------
\begin{scope}[shift={(4.2,0)}]
  \node[panel] (RightBox) at (0,0) {};
  \node[gray!70] at (0,2.5) {\bfseries Nontrivial 1-cycle (not a boundary)};

  % Background “sheet”
  \fill[gray!06] (-2.8,-2.0) rectangle (2.8,2.0);

  % A hole (missing 2-chain)
  \node[holedisc] (Hole) at (0,0) [circle, minimum width=2.0cm, minimum height=2.0cm] {};

  % cycle going around the hole (cannot bound any 2-chain in this space)
  \draw[cyclenontriv] (0,0) circle (1.4);

  % Labels
  \node[gray!70] at (0,0) {\scriptsize hole};
  \node[red!70] at (0,-1.8) {$[\gamma]\neq 0 \ \text{in}\ H_1$};
\end{scope}

\end{tikzpicture}
\caption{\textbf{Trivial vs.\ nontrivial 1-cycles.} 
\emph{Left:} A cycle that is the boundary of a filled region ($\gamma=\partial S$) is 
\emph{null-homologous} and hence trivial in $H_1$: it can be “canceled” as a boundary. 
\emph{Right:} A cycle encircling a hole is not the boundary of any 2-chain in the space, 
so it represents a \emph{nontrivial} class in $H_1$. In our framework, trivial cycles 
correspond to high-entropy, short-lived scaffolds ($\Psi$) that collapse under boundary 
cancellation ($\partial^2=0$), whereas nontrivial cycles correspond to low-entropy content 
invariants ($\Phi$) that persist as memory.}
\label{fig:trivial-vs-nontrivial}
\end{figure}

\begin{theorem}[Cycles Encode Order Invariance]
\label{thm:order-invariance}
Let $(\mathcal{Z},x_0)$ be a pointed state space (latent manifold or graph) with base state $x_0$
(``home''). Let $\mathcal{A}=\{a_1,\dots,a_m\}$ denote a finite set of local moves inducing paths
$\{\alpha_i\}$ starting and ending in a neighborhood of their endpoints.
For any finite sequence of moves $w=a_{i_1}\cdots a_{i_k}$ that yields a cycle
$\gamma_w$ at $x_0$ (i.e., a homing trajectory), the first homology class
$[\gamma_w]\in H_1(\mathcal{Z};\mathbb{Z})$ depends only on the \emph{multiset} of moves used (and their net orientations), not on their order.
Equivalently, all order permutations of $w$ that remain cycles at $x_0$ determine the same element in $H_1$.
\end{theorem}

\begin{proof}[Proof:]
The key insight is the Abelian property of the addition operator. Concatenate local moves to form cycles based at $x_0$, producing elements of the fundamental group
$\pi_1(\mathcal{Z},x_0)$. The Hurewicz map $h:\pi_1(\mathcal{Z},x_0)\to H_1(\mathcal{Z};\mathbb{Z})$
abelianizes path composition: commutators vanish in $H_1$. Hence for cycles $\gamma,\eta$,
$[\gamma\cdot\eta]=[\eta\cdot\gamma]$ and, more generally, any permutation of cycle segments yields the same homology class,
provided the path remains closed. Thus $[\gamma_w]$ is invariant to the \emph{order} of constituent moves and depends only
on their cumulative 1-chain (the signed sum of traversed edges/segments). Intuitively, homology collapses
all order-specific reparameterizations and commutator structure, retaining only the closed-cycle content.
\end{proof}

\begin{example}[Dots as $H_0$ vs.\ cycles as $H_1$ in Memory Consolidation]
\label{example:dots-vs-cycles}
In the homological interpretation of memory, trivial and nontrivial cycles 
map onto distinct cognitive outcomes:
1) (\emph{Dots as $H_0$}) Collapsed or trivial cycles reduce to isolated 
points in the state space. Algebraically, these correspond to connected 
components $H_0(\mathcal{Z})$, encoding mere existence without relational 
structure. Cognitively, such fragments represent unbound or forgotten traces, 
which fail to persist across perturbations.
2) (\emph{cycles as $H_1$}) Nontrivial closed trajectories correspond to 
classes in $H_1(\mathcal{Z})$, encoding relations that cannot be reduced to 
boundaries. These cycles serve as persistent invariants, forming the low-entropy 
content variables $\Phi$ that underlie stable memory and predictive generalization.
\end{example}

Theorem~\ref{thm:order-invariance} establishes that once trajectories are 
organized into cycles, their predictive value no longer depends on the precise 
ordering of steps but only on the closure of the cycle. This reduction reflects 
a deeper topological dichotomy in memory formation. Algebraically, the identity 
$\partial^2=0$ ensures that boundaries of boundaries vanish: incomplete chains 
cannot accumulate meaning unless they close, and only closed cycles can survive 
as invariants. Cognitively, this corresponds to the fact that exploratory 
fragments either collapse into trivial points (dots) with no relational content, 
or are stabilized into nontrivial cycles that encode order-invariant memory. In 
this sense, $\partial^2=0$ acts as the algebraic filter that separates forgotten 
scaffolds from consolidated invariants. To make this distinction explicit, we now 
formalize the roles of $H_0$ and $H_1$ in memory consolidation (refer to Fig. \ref{fig:trivial-vs-nontrivial}).

\begin{lemma}[$\partial^2=0$ Enforces the Dot-cycle Dichotomy]
\label{lemma:dot-cycle}
Let $C_\ast(\mathcal{Z})$ denote the chain complex of a neural state space $\mathcal{Z}$. 
The homological identity $\partial^2=0$ implies that:
1) Any open chain $\sigma \in C_1(\mathcal{Z})$ with $\partial \sigma \neq 0$ 
must collapse to a trivial 0-cycle in $H_0(\mathcal{Z})$, encoding mere connectivity 
without relational content.
2) Any closed chain $\gamma \in C_1(\mathcal{Z})$ with $\partial \gamma = 0$ 
defines a homology class $[\gamma]\in H_1(\mathcal{Z})$. If $\gamma$ is not the 
boundary of a higher-dimensional chain, it represents a nontrivial cycle that 
persists as a stable memory trace.
Thus, $\partial^2=0$ acts as a topological filter: boundaries of boundaries vanish, 
ensuring that only two outcomes are possible, collapse into trivial dots ($H_0$) 
or persistence as nontrivial cycles ($H_1$). In cognition, this dichotomy captures 
the distinction between forgotten fragments and consolidated memories.
\end{lemma}

Lemma~\ref{lemma:dot-cycle} formalizes the algebraic consequence of 
$\partial^2=0$: boundaries of boundaries vanish, leaving only two 
possibilities for 1-chains, collapse into trivial 0-cycles or persistence 
as nontrivial 1-cycles, which are analogous to binary states in classical Turing machines. To connect this structural dichotomy with cognition, we now interpret these two outcomes in terms of memory consolidation: 
dots in $H_0$ correspond to forgotten fragments, while cycles in $H_1$ 
correspond to persistent, order-invariant memories. 

\paragraph{From dots to cycles as order-invariant memory.}
Example~\ref{example:dots-vs-cycles} illustrates how trivial $H_0$ components 
(``dots'') capture disconnected, high-entropy fragments, while nontrivial $H_1$ 
cycles encode stable relational invariants that persist under perturbation. 
Theorem~\ref{thm:order-invariance} then formalizes this intuition: once a 
trajectory closes into a cycle, its homology class is invariant under 
permutations of the generating moves, depending only on their multiset and 
orientation. In other words, loops erase superficial ordering while preserving 
relational content. Figure~\ref{fig:memory} situates this process within the 
CCUP pipeline: high-entropy scaffolds ($\Psi$) consist of transient loops and 
dangling boundaries; oscillatory alignment and coincidence detection enforce 
closure ($\partial^2=0$), thereby transforming dots into cycles; and the result 
is low-entropy content ($\Phi$) stabilized as persistent homology classes. 
This transition from dots to cycles is the algebraic mechanism of predictive coding via inverted inference \cite{li2025Inverted}, ensuring that cognition retains order-invariant relational 
structures rather than fragile, order-sensitive fragments. Formally, we have

\begin{lemma}[Inverted Inference Facilitates Closure]
Let $(\Psi,\Phi)$ denote context-content pairs in a chain complex 
$(C_\bullet,\partial)$. Forward inference projects $\Psi\to\Phi$, 
but may leave residual boundaries $\partial(\Psi,\Phi)\neq 0$. 
Inverted inference updates $\Psi \mapsto \Psi'$ such that 
mismatched boundaries are absorbed, yielding:
$\partial^2(\Psi',\Phi)=0$.
Hence, inverted inference acts as a homotopy operator that cancels 
misaligned boundaries, ensuring that only reproducible cycles 
$[\gamma]\in H_k(C_\bullet)$ survive. 
\end{lemma}

%\paragraph{Topology to information.}
The homological account above identifies memory as order-invariant cycles
produced by closure and stabilized by inverted inference. To view the same
mechanism through a complementary lens, we translate “closure” into
information-theoretic terms: cycles correspond to \emph{broken symmetries} that
reduce uncertainty by concentrating probability mass on reproducible
configurations, thereby enabling prediction. In this sense, topological
persistence (nontrivial $H_k$ classes) and statistical efficiency (entropy
reduction) are two descriptions of the same phenomenon, local mismatches are
cancelled as boundaries, and what remains are low-entropy, order-invariant
carriers of structure. We now formalize this correspondence.

\subsection{Information-Theoretic View: Entropy Minimization and Dynamic Alignment}

The preceding results established invariant cycles as the structural backbone 
of non-ergodic dynamics, grounding memory persistence and order invariance in 
homology. Complementary to this topological account, one can adopt an 
information-theoretic perspective \cite{cover1999elements}. From this viewpoint, symmetry corresponds 
to maximal entropy and minimal predictive power, while symmetry breaking 
reduces entropy and concentrates probability mass onto invariant cycles. 
The following lemmas formalize this connection, showing how entropy 
minimization through cycle formation provides a substrate for reliable and 
order-independent prediction.

\begin{lemma}[Broken Symmetry Facilitates Prediction via Entropy Minimization]
\label{lem:entropy}
Let a system evolve on a state space $\mathcal{Z}$ with symmetry group $G$.
If all $g \in G$ are equivalent, the induced distribution over outcomes is uniform,
yielding maximal entropy $H = \log |G|$ and no predictive power.
A perturbation breaking $G$ reduces the support of admissible outcomes,
lowering entropy $H' < H$.
The residual invariant cycle $\gamma \subset \mathcal{Z}$ preserves what remains
stable across variations, providing a reliable substrate for prediction.
Thus, prediction becomes possible when symmetry is broken, because uncertainty
is reduced and invariants encode persistent, order-independent structure.
\end{lemma}

The first lemma establishes that prediction becomes possible only after symmetry 
is broken, since this reduces entropy and reveals residual invariants. 
Yet the existence of invariants alone does not guarantee predictive efficiency: 
the system must further \emph{organize} these invariants into closed cycles 
that filter out order-dependent noise \cite{khona2022attractor}. In other words, broken symmetry provides 
the conditions for prediction, while the concentration of dynamics onto invariant 
cycles ensures that prediction is \emph{robust} and generalizable. 
The next lemma formalizes this refinement.

\begin{proposition}[Entropy Minimization Improves Prediction by Cycles]
\label{prop:prediction-simple}
Let a system generate trajectories in a state space $\mathcal{Z}$.
Suppose initially, the system has a symmetry $G$ (e.g.\ different orders of moves or observations are treated as equivalent). A perturbation breaks this full symmetry, but leaves behind
an invariant cycle $\gamma \subset \mathcal{Z}$ with $\partial \gamma = 0$.
Then we have:
1) The cycle $\gamma$ encodes what is stable across different orders or paths;
2) Predictions about future outcomes need only depend on $\gamma$ (and context),
not on the detailed order of past steps;
3) Thus, broken symmetry reduces noise from order-specific variations and improves
prediction by preserving only what remains invariant.
\end{proposition}

The previous lemma establishes how, at the local level of trajectories,
symmetry breaking reduces entropy by filtering away order-dependent noise,
leaving behind invariant cycles that encode predictive regularities. 
To extend this observation to the global dynamics of a non-ergodic system,
we must recognize that prediction is reliable only when the evolving 
probability measure itself concentrates on such cycles. 
This concentration formalizes the shift from transient fluctuations 
to persistent invariants, yielding the following corollary.

\begin{corollary}[Prediction as Concentration on Cycles]
\label{cor:prediction}
For a non-ergodic system $(X,T)$, prediction is possible if and only if 
the probability measure $\mu_t$ concentrates on invariant cycles 
$[\gamma] \in H_k(X)$ as $t \to \infty$. 
Equivalently:
\[
\text{Prediction} \;\;\Longleftrightarrow\;\;
\text{Entropy Reduction by Symmetry Breaking} \;\;\Longleftrightarrow\;\;
\text{Measure Concentration on Cycles}.
\]
Thus, the structural invariants of broken symmetry are precisely the 
carriers of predictive information, ensuring reliability of memory and 
generalization across time.
\end{corollary}

The corollary identifies prediction with the global concentration of measure 
on invariant cycles, showing that reliability emerges only when dynamics 
collapse onto such persistent structures. 
To understand how these cycles arise in the first place, we return to the 
local mechanism: symmetry breaking. 
Whenever a perturbation forces the system to resolve among equivalent 
alternatives, the act of breaking symmetry does not erase the discarded 
possibilities but organizes them into a recurrent cycle. 
This cycle both stabilizes the chosen outcome and retains access to its 
counterfactuals, thereby generating the invariant cycles that ultimately 
carry predictive information.
The information-theoretic account emphasizes how symmetry breaking and entropy 
minimization funnel dynamics onto invariant cycles, thereby transforming 
uncertainty into predictive stability. Yet, this description remains at the level 
of abstract measures and distributions. 

\paragraph{Content variable $\Phi$ as low-entropy homology.}
Within CCUP, the content variable $\Phi$ corresponds to information that is 
both specific and stable. Mathematically, $\Phi$ is identified with 
nontrivial homology classes: cycles $[\gamma] \in H_k(\mathcal{Z})$ that 
cannot be reduced to boundaries. Such cycles encode persistent, 
low-entropy structures because many possible trajectories or micro-states 
collapse into the same equivalence class. In neural terms, $\Phi$ reflects 
patterns of activity that recur reliably across different contexts, such as 
a learned motor primitive, a familiar spatial route, or a well-established 
object representation. By filtering away order-dependent variability, 
$\Phi$ preserves only the invariant relational structure that remains after 
symmetry breaking. This makes $\Phi$ the stable substrate of memory and the 
carrier of predictive power: once identified, it can be recalled, reused, 
and composed into higher-order cognitive structures.

\paragraph{Context variable $\Psi$ as high-entropy scaffolding.}
In contrast, the context variable $\Psi$ captures the transient, exploratory, 
and often noisy aspects of cognition. Topologically, $\Psi$ is associated with 
trivial cycles or short-lived features in the persistence barcode: loops that 
quickly vanish under perturbation or deformation. These cycles act as 
\emph{scaffolding}, supporting the discovery and stabilization of $\Phi$ but 
not themselves persisting as memory. In information-theoretic terms, 
$\Psi$ is high-entropy: it reflects a large space of possibilities, many of 
which will be pruned away as the system concentrates its measure on 
low-entropy $\Phi$ structures. Biologically, $\Psi$ is implemented by 
slow, contextual rhythms (e.g.\ theta oscillations) or exploratory neural 
activity that supplies diverse scaffolds for binding. Through dynamic 
alignment and phase-resetting, these high-entropy contextual structures are 
folded into persistent content loops, allowing cognition to maintain 
flexibility while ensuring stability in memory formation.

Taken together, these complementary roles highlight the dual nature of CCUP (refer to Fig. \ref{fig:memory}): 
$\Phi$ encodes stability through nontrivial cycles, while $\Psi$ supplies the 
flexibility of transient scaffolds. This duality between $\Phi$ and $\Psi$ can be formalized as an ordering 
principle: stable low-entropy structures must emerge first, while 
high-entropy contextual scaffolds provide the variability that refines 
and specifies them, leading to the following theorem.

\begin{principle}[Structure-Before-Specificity Principle under CCUP]
\label{prin:structure-specificity}
Let $\Phi$ denote low-entropy content variables corresponding to 
nontrivial homology classes $[\gamma]\in H_k(\mathcal{Z})$, and let 
$\Psi$ denote high-entropy contextual scaffolds corresponding to 
transient or trivial cycles. Then cognition obeys the following principle:
1) (\textbf{Structure before specificity}) Stable content $\Phi$ 
arises from nontrivial cycles that persist across perturbations. 
These cycles define the backbone of memory and predictive power.
2) (\textbf{Specificity from scaffolding}) Context $\Psi$ supplies 
a high-entropy exploratory substrate: transient cycles that may 
collapse but provide the variability needed to refine, adapt, or 
recombine $\Phi$.
3) (\textbf{Dynamic alignment}) The interaction of $\Psi$ and $\Phi$ 
via cycle closure ($\partial^2=0$) ensures that contextual exploration 
is funneled into persistent content loops, transforming noisy scaffolds 
into stable memory traces.
\end{principle}

In this formulation, cognition develops by first establishing stable 
low-entropy homology classes ($\Phi$), upon which higher-entropy contextual 
scaffolds ($\Psi$) can later impose specificity. This ordering guarantees 
that adaptability is grounded in persistent structure rather than in 
transient variability alone. The crucial step is to explain how these 
two layers, persistent content cycles and transient contextual scaffolds, are 
dynamically aligned in a way that preserves stability while allowing 
flexibility. Topological closure offers the key: by enforcing the identity 
$\partial^2=0$, closure ensures that exploratory scaffolds do not remain 
open-ended fluctuations but are either pruned away or collapsed into 
well-defined loops that contribute to memory. In other words, closure 
provides a homological filter that converts noisy, high-entropy scaffolding 
into reproducible, low-entropy cycles of content. The following proposition 
makes this alignment explicit by showing how biological systems implement 
closure through oscillatory phase coding and coincidence detection, thereby 
realizing the CCUP principle in neural dynamics.

\begin{proposition}[Biological Implementation of Dynamic Alignment]
In CCUP, content variables $\Phi$ correspond to nontrivial homology classes 
(persistent cycles), while context $\Psi$ corresponds to transient or trivial cycles. 
Biologically, dynamic alignment of $\Psi$ and $\Phi$ is realized by:
1) Oscillatory phase coding, where slow rhythms (e.g.\ theta) provide contextual scaffolds and fast rhythms (e.g.\ gamma) encode content, with alignment achieved through phase locking;
2) Coincidence detection, where neurons fire only when contextual and content inputs arrive synchronously, 
enforcing boundary cancellation ($\partial^2=0$) and stabilizing closed loops.
Thus, neural dynamics implement CCUP by aligning transient contextual scaffolds with persistent 
content cycles, ensuring memory persistence and predictive efficiency.
\end{proposition}

% ---------- TikZ: CCUP (Context Ψ → Alignment → Content Φ) ----------
\begin{figure*}[t]
\centering
\begin{tikzpicture}[
  >=Latex,
  scale=1.0,
  every node/.style={font=\small},
  box/.style={rounded corners=6pt, draw=gray!60, line width=0.8pt, minimum width=4.6cm, minimum height=4.2cm},
  title/.style={font=\small\bfseries, text=gray!70},
  faintloop/.style={gray!55, line width=0.7pt},
  strongloop/.style={blue!70, line width=1.2pt},
  alignarr/.style={-Latex, thick, gray!65},
  win/.style={line width=2.2pt, green!60!black, line cap=round},
  barcode/.style={line width=2.0pt},
  barfaint/.style={gray!60, barcode},
  barstrong/.style={blue!70, barcode}
]

% Layout anchors
\coordinate (A) at (-6.1,0);
\coordinate (B) at (0,0);
\coordinate (C) at (6.1,0);

% ---------- Entropy meters (top) ----------
\draw[gray!55, line width=0.6pt] (-7.6,2.9) -- (7.6,2.9);
\node[title] at (-0.2,3.25) {entropy};
% gradient arrows
\draw[->, thick, red!70] (-7.4,3.0) -- (-3.8,3.0) node[midway, above] {\scriptsize high};
\draw[->, thick, orange!70] (-3.8,3.0) -- (-0.2,3.0);
\draw[->, thick, green!70!black] (-0.2,3.0) -- (3.4,3.0);
\draw[->, thick, blue!70] (3.4,3.0) -- (7.0,3.0) node[midway, above] {\scriptsize low};

% ---------- Left Panel: Context Ψ (high-entropy scaffolds) ----------
\node[box, minimum width=5.2cm] (Left) at (A) {};
\node[title] at ([yshift=2.4cm]A) {$\Psi$ \, (context): high-entropy scaffolds};

% Many faint, short-lived loops
\foreach \x/\y/\r in {-7.7/0.8/0.5, -6.7/1.2/0.7, -5.5/0.5/0.6, -6.0/-0.9/0.7, -7.0/-0.4/0.5, -5.1/1.0/0.5, -6.4/0.1/0.4}{
  \draw[faintloop] (\x,\y) .. controls (\x+\r, \y+0.3) and (\x-\r, \y-0.2) .. (\x,\y);
}
% A few broken chains (open curves) to suggest non-closure
\draw[faintloop] (-5.2,-0.2) .. controls (-5.0,0.2) .. (-4.8,0.0);
\draw[faintloop] (-7.3,0.1) .. controls (-7.0,-0.2) .. (-6.7,-0.1);

% Mini persistence barcode (short bars)
\draw[barfaint] (-7.5,-1.7) -- (-7.5,-1.3);
\draw[barfaint] (-7.1,-1.7) -- (-7.1,-1.45);
\draw[barfaint] (-6.7,-1.7) -- (-6.7,-1.35);
\draw[barfaint] (-6.3,-1.7) -- (-6.3,-1.5);
\draw[barfaint] (-5.9,-1.7) -- (-5.9,-1.4);
\draw[barfaint] (-5.5,-1.7) -- (-5.5,-1.55);
\draw[barfaint] (-5.1,-1.7) -- (-5.1,-1.45);
\node[gray!60] at (-6.3,-2.2) {\scriptsize persistence barcode (transient)};

% ---------- Middle Panel: Alignment (phase + coincidence) ----------
\node[box, minimum width=5.2cm] (Mid) at (B) {};
\node[title] at ([yshift=2.4cm]B) {Alignment: phase coding $+$ coincidence ($\partial^2\!=\!0$)};

% Theta ring
\def\R{1.4}
\draw[gray!65, line width=1.0pt] (0,0.2) circle (\R);
\node[gray!70] at (0,1.95) {\scriptsize $\theta$ phase on $S^1$};

% Binding window
\def\a{30} \def\b{90}
\draw[win] ({\R*cos(\a)},{0.2+\R*sin(\a)}) arc (\a:\b:\R);

% Gamma ticks at phases
\foreach \ang/\lbl in {10/$g_1$, 60/$g_2$, 140/$g_3$, 220/$g_4$, 300/$g_5$}{
  \draw[blue!70, line width=1.1pt] ({0.8*\R*cos(\ang)},{0.2+0.8*\R*sin(\ang)}) --
                                   ({0.98*\R*cos(\ang)},{0.2+0.98*\R*sin(\ang)});
}

% Coincidence window box label
\node[green!60!black] at ({0.2+1.6*\R*cos(60)},{-0.2+1.6*\R*sin(60)}) {\scriptsize coincidence window $\Delta$};

% Closure arrow around circle
\draw[->, gray!65, line width=0.9pt] (0,-1.5) .. controls (2.0,-1.7) .. (0,1.8);
\node[gray!65] at (2.1,-1.6) {\scriptsize closure $\Rightarrow~\partial^2=0$};

% ---------- Right Panel: Content Φ (low-entropy homology) ----------
\node[box, minimum width=5.2cm] (Right) at (C) {};
\node[title] at ([yshift=2.4cm]C) {$\Phi$ \, (content): low-entropy homology};

% Few strong closed loops (persistent cycles)
\draw[strongloop] (4.8,0.9) .. controls (5.4,1.5) and (6.0,0.3) .. (5.5,0.0) .. controls (5.0,-0.3) and (4.6,0.3) .. (4.8,0.9);
\draw[strongloop] (6.2,0.6) .. controls (6.7,1.2) and (7.3,0.7) .. (7.2,0.1) .. controls (7.1,-0.5) and (6.4,-0.2) .. (6.2,0.6);

% Mini persistence barcode (long bars)
\draw[barstrong] (4.7,-1.7) -- (4.7,-1.0);
\draw[barstrong] (5.2,-1.7) -- (5.2,-0.9);
\draw[barstrong] (5.7,-1.7) -- (5.7,-0.85);
\draw[barstrong] (6.2,-1.7) -- (6.2,-0.95);
\draw[barstrong] (6.7,-1.7) -- (6.7,-0.9);
\node[blue!70] at (6.0,-2.2) {\scriptsize persistence barcode (invariants)};

% ---------- Arrows between panels ----------
\draw[alignarr] (-3.5,-0.1) -- (-1.3,-0.1) node[midway, above] {\scriptsize phase-resetting};
\draw[alignarr] (1.3,-0.1) -- (3.5,-0.1) node[midway, above] {\scriptsize boundary cancellation};

\end{tikzpicture}
\caption{\textbf{Memory consolidation via cycle closure: Context $\Psi$ to Content $\Phi$ via Cyclic Alignment.}
\textbf{Left:} High-entropy context ($\Psi$) consists of many transient, order-dependent
loops (short persistence bars). \textbf{Middle:} Oscillatory phase coding (theta ring)
and coincidence detection (window $\Delta$) align events by phase and enforce
closure ($\partial^2=0$). \textbf{Right:} Low-entropy content ($\Phi$) emerges as a
small set of persistent cycles (long bars), i.e., nontrivial homology classes that
serve as memory-bearing invariants.}
\label{fig:memory}
\end{figure*}
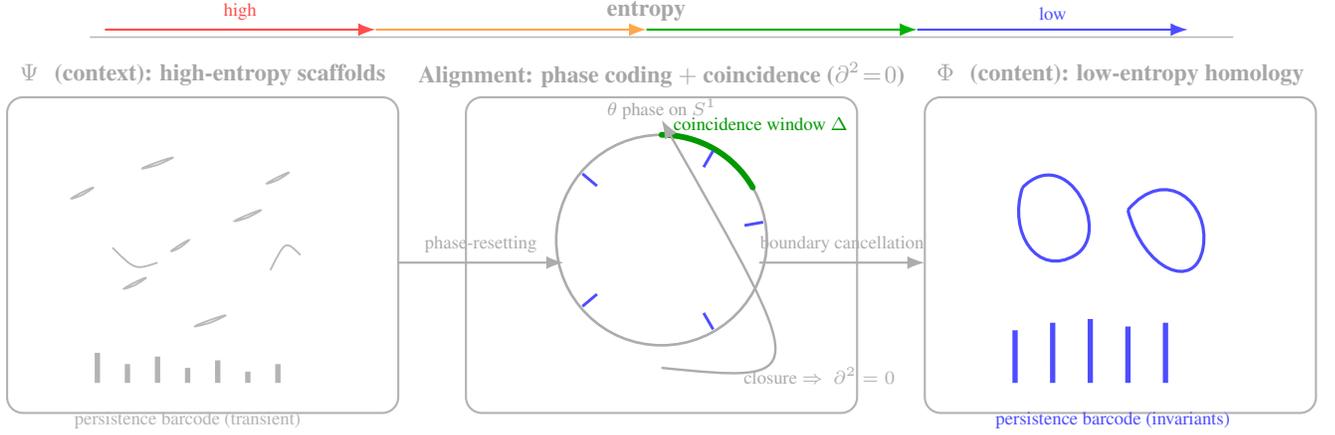

\noindent\textbf{Remark:}
Memory consolidation implements a topological dichotomy: transient 
experiences collapse into trivial 0-cycles (dots), while meaningful and 
predictive structures survive as nontrivial 1-cycles (cycles).
This dichotomy reflects a key insight of the context-content uncertainty principle (CCUP) \cite{li2025CCUP}: 
\textbf{nontriviality in homology encodes the low-entropy content variable 
$\Phi$, while trivial cycles may reflect transient contextual scaffolding 
$\Psi$}. In other words, persistence in $H_1$ corresponds to content that 
survives consolidation and supports prediction, whereas it collapses into $H_0$ 
marks the fate of high-entropy scaffolds that fail to stabilize. This 
interpretation situates the dot-cycle distinction not merely as a topological 
artifact, but as the algebraic expression of how cognition balances 
adaptability (through high-entropic $\Psi$) and stability (through low-entropic $\Phi$).

To understand how dynamic alignment is 
realized in neural systems, we must enrich the framework with topological and 
biological structure. Polychronous neural groups (PNGs) provide the substrate: 
they are defined by precise, delay-locked spike-timing patterns that naturally 
form cycles in neural state space. Their organization rests on two complementary 
mechanisms. First, oscillatory phase coding supplies a cyclic parameterization of 
time, embedding spikes into recurrent cycles and providing the contextual scaffold 
($\Psi$). Second, coincidence detection enforces boundary cancellation by 
filtering out misaligned spikes and stabilizing only those closed trajectories 
that persist as content cycles ($\Phi$). Together, these mechanisms instantiate 
the algebraic principle $\partial^2=0$ in the brain: oscillations define the 
circle on which time is wrapped, and coincidence detection prunes away dangling 
boundaries, ensuring that only invariant cycles survive. The next section develops 
this \emph{information-topological framework} for biological cycle formation, 
showing how PNGs, through phase coding and coincidence detection, implement the 
core insight that cognition arises from the dynamic alignment of context and 
content in nested, hierarchical cycles.

\section{Information Topological Framework for Scaffolding and Closure}
\label{sec:3}

%\paragraph{Temporal closure in neural dynamics.}
Oscillations discretize time on a circle ($S^1$), providing phase bins within which coincidence detection collapses fragments into recurrent traversals. 
Mathematically, the boundary calculus enforces this filtration: $\partial^2=0$ cancels unmatched endpoints so that only closed chains survive as persistent cycles. 
Cognitively, isolated tokens (dots) do not stabilize memory; only when linked by contextual relations into cycles do they consolidate as durable traces.
In this section, we show how oscillatory phase coding and coincidence detection implement temporal scaffolding and boundary cancellation in spiking networks, turning temporal fragments into cycles. 
%We then develop an \emph{Information Topological Framework for Hierarchical Cycle Generation} to show how polychronous neural groups, via conduction delays and plasticity, stitch local closures across timescales (e.g., $\gamma \!\to\! \theta \!\to\!$ ripple), yielding nested, persistent cycles that realize hierarchical memory.

\subsection{Oscillation Phase Coding as Temporal Scaffolding}

%\paragraph{Oscillatory phase coding as scaffold ($\Psi$).}
Neural oscillations instantiate the closure principle by quotienting linear time to a circle: an oscillator implements $t \mapsto e^{i\omega t}\in S^1$, so events are registered by \emph{phase} rather than absolute time.
Biologically, this scaffold is realized at multiple, coupled timescales. 
(i) \emph{Theta–gamma nesting} (e.g., hippocampus–entorhinal) provides a macrocycle ($\theta$, 4–12 Hz) that segments experience and a microcycle ($\gamma$, 30–100 Hz) that tiles each $\theta$ bin with ordered subevents; phase–amplitude coupling thus lays out a toroidal code $S^1_\theta \times S^1_\gamma$ in which winds index recurrent cycles \cite{LismanJensen2013,CanoltyKnight2010,Canolty2006PNAS}.
(ii) \emph{Coincidence detection} sharpens edges of these cycles: NMDA nonlinearity, backpropagating spikes, and fast interneuron circuitry (PV/ING, PING) create narrow $\mathcal{O}(1\!-\!10\text{ ms})$ windows so that only spikes aligned within a phase bin form effective synaptic links; misaligned fragments fail to bind and are pruned \cite{KonigEngelSinger1996,StuartSakmann1994,BuzsakiWang2012}.
(iii) \emph{Spike-timing dependent plasticity (STDP)} orients these links by phase lead/lag, turning phase offsets into directed edges in a chain; repeated traversal within a cycle consolidates these edges, canceling stray endpoints and favoring closed walks \cite{Markram1997Science,BiPoo1998,CaporaleDan2008}.
(iv) \emph{Conduction delays and myelin plasticity} tune effective phase lags, enabling polychronous assemblies: axonal/dendritic delays align distributed spikes into reproducible phase patterns that complete cycles despite spatial dispersion \cite{Izhikevich2006,PajevicBasserFields2014,Fields2015NRN}.
(v) \emph{Phase-of-firing coding and precession} (e.g., hippocampal place cells) map position or task progress to phase on $S^1_\theta$, so that a behavioral episode corresponds to a return map on the Poincaré section; complete laps close in phase space, incomplete traversals do not \cite{OKeefeRecce1993,Montemurro2008CB}.
(vi) \emph{State-dependent reentry} (sharp-wave ripples during NREM/quiet wake) replays phase-ordered sequences on a faster carrier, tightening weights along already-closed paths and suppressing nonclosing detours \cite{FosterWilson2006,DibaBuzsaki2007}.

\noindent\emph{Interpretation.}
Oscillations supply the contextual scaffold $\Psi$ that folds timelines into cyclic coordinates; coincidence and plasticity then implement boundary cancellation in synaptic space.
What persists are cycles, phase-locked traversals whose endpoints identify on $S^1$, while unmatched fragments dissipate.
This sets up the formal lemma below, which recasts phase-binned spiking as a chain whose boundary vanishes after a full cycle.

\begin{lemma}[Oscillatory Phase Coding as Temporal Scaffolding]
\label{lem:phase-closure}
Let $\theta(t) = \omega t \pmod{2\pi}$ denote the phase of a neural oscillator, 
with events encoded relative to $\theta(t)$ on the circle $S^1$. Then oscillatory 
phase coding induces the following invariants:
1) \textbf{Binding:} Events occurring within the same phase window 
    $\theta(t) \in [\phi, \phi+\Delta]$ are grouped together, forming a coherent 
    representation;
2) \textbf{Ordering:} Sequences of events are represented by their relative 
    phase offsets $(\Delta \theta_1,\Delta \theta_2,\dots)$, embedding linear order 
    into a cyclic scaffold;
3) \textbf{Closure:} After a full cycle $\theta(t+T)=\theta(t)$ with 
    $T=\tfrac{2\pi}{\omega}$, the system resets, ensuring that trajectories are 
    organized into cycles rather than unbounded chains.
Together, these properties enforce the topological identity 
$\partial^2 = 0$ at the temporal level: the boundary of one temporal segment 
becomes the beginning of the next, so that each cycle closes before a new one 
begins. Consequently, oscillatory phase coding guarantees consistency of 
memory traces by embedding them in recurrent temporal cycles.
\end{lemma}

The formal statement of Lemma~\ref{lem:phase-closure} captures how oscillatory 
phase coding transforms linear time into a cyclic scaffold, guaranteeing binding, 
ordering, and closure. To visualize this principle, 
Fig.~\ref{fig:phase-coding-closure} illustrates how linear time $t$ is wrapped 
onto the circle $S^1$ (theta phase), with discrete gamma packets embedded at 
distinct phases. Events that fall into the same phase window (green arc) are 
bound together, while relative phase offsets encode ordering. The reset at the 
end of each $\theta$ cycle ensures closure, embodying the algebraic identity 
$\partial^2=0$ in biological timekeeping.
%\noindent\textbf{Closure (expanded).}
Let $\theta:\mathbb{R}\!\to\!S^1$ be the phase map $\theta(t)=\omega t \!\!\mod 2\pi$, so $T=\tfrac{2\pi}{\omega}$ identifies $t \sim t+T$ and quotients linear time to a circle. 
Partition $S^1$ into $L$ phase bins $\{\varphi_\ell\}_{\ell=1}^L$ and let $v_\ell$ denote the (phase-binned) latent state aggregated within bin $\varphi_\ell$. 
Define oriented edges $e_\ell=[v_\ell,v_{\ell+1}]$ with $v_{L+1}\equiv v_1$. 
The phase-ordered chain $c=\sum_{\ell=1}^L e_\ell$ has
$\partial c=\sum_{\ell=1}^L (v_{\ell+1}-v_\ell)=v_{L+1}-v_1=0$,
so a full $2\pi$ sweep closes into a $1$–cycle. Coincidence detection enforces this construction: only events aligned within a phase window of width $\varepsilon$ create edges, pruning stray fragments whose endpoints would otherwise fail to cancel. Conduction delays implement modular jumps $e_\ell=[v_\ell,v_{\ell+k}]$, yielding a winding number $k$ on $S^1$; after $L$ such steps the path returns to $v_1$, again giving $\partial c=0$ and a homology class $[c]\in H_1(S^1)\cong\mathbb{Z}$.

Two useful views follow. (i) \emph{Poincaré/return map:} sampling at phase $\phi_0$ defines $F_{\phi_0}:x(t)\!\mapsto\!x(t+T)$; fixed points and periodic points of $F_{\phi_0}$ are closed orbits on $S^1$, i.e., cycles \cite{poincare_map_neural}. (ii) \emph{Cross-frequency nesting:} with $\theta$ and $\gamma$ phases, time quotients to a torus $S^1_\theta\!\times\! S^1_\gamma$ and $H_1\!\cong\!\mathbb{Z}^2$; winds $(k_\gamma,k_\theta)$ encode hierarchical cycles \cite{theta_gamma_cfc_memory}. 
\emph{Summary.} Phase coding turns linear sequences into cyclic invariants: the absolute start/end times are identified on $S^1$, so boundaries telescope away and only closed traversals persist. Coincidence gates which edges exist; $\partial^2\!=\!0$ guarantees unmatched endpoints cannot accumulate into memory, while completed cycles survive as stable traces.

% ---------- TikZ: Oscillatory Phase Coding (theta circle + gamma packets) ----------
\begin{figure}[t]
\centering
\begin{tikzpicture}[
  >=Latex,
  scale=1.05,
  every node/.style={font=\small},
  th/.style={line width=1.1pt, gray!70},
  gpack/.style={line width=0.9pt, -{Latex[length=2.0mm]}},
  bindarc/.style={line width=2.2pt, green!60!black, line cap=round},
  guide/.style={gray!55, dashed, line width=0.6pt}
]

% --- Panel A: Time -> S^1 mapping -------------------------------------------------
% Time axis
\draw[->,th] (-5.6,0) -- (-2.2,0) node[right] {$t$ (linear time)};
% Sample ticks t0..t4
\foreach \x/\lab in {-5.2/$t_0$,-4.7/$t_1$,-4.2/$t_2$,-3.7/$t_3$,-3.2/$t_4$}{
  \draw (\x,0.12) -- (\x,-0.12) node[below=1.3mm] {\lab};
}
% Arrows mapping times to phase points on circle
% Circle center for panel B will be at (2.8,0); for panel A, draw mapping arrows
\def\R{2.2} \def\cx{2.8} \def\cy{0}
% Helper: map 5 times to 5 phases
\foreach \x/\ang in {-5.2/140,-4.7/60,-4.2/330,-3.7/250,-3.2/180}{
  \draw[->,gray!60] (\x,0.35) .. controls (-1.9,1.2) .. ({\cx+1.72*cos(\ang)},{\cy+1.72*sin(\ang)});
}
\node[align=left] at (-4,2.) {\footnotesize Time wrapped onto $S^1$:\\[-1pt]\footnotesize $t \mapsto e^{i\omega t}$};

% --- Panel B: Theta phase ring (S^1) and gamma packets ----------------------------
% Theta ring
\draw[th] (\cx,\cy) circle (\R);
\node[gray!70] at (\cx,\cy+\R+0.5) {$\theta$ (phase on $S^1$)};

% Phase tick marks (0, 90, 180, 270 deg)
\foreach \ang/\lab in {0/0,90/\frac{\pi}{2},180/\pi,270/\frac{3\pi}{2}}{
  \draw[gray!50, line width=0.5pt] ({\cx+1.04*\R*cos(\ang)},{\cy+1.04*\R*sin(\ang)}) --
                                  ({\cx+0.96*\R*cos(\ang)},{\cy+0.96*\R*sin(\ang)});
  \node[gray!60] at ({\cx+1.28*\R*cos(\ang)},{\cy+1.28*\R*sin(\ang)}) {\scriptsize $\lab$};
}

% Binding window (thick green arc)
\def\a{35} \def\b{85} % phase window [a,b]
\draw[bindarc] ({\cx+\R*cos(\a)},{\cy+\R*sin(\a)}) arc (\a:\b:\R);
\node[green!60!black] at ({\cx+1.35*\R*cos(60)},{\cy+1.35*\R*sin(60)}) {\scriptsize binding window};

% Gamma packets at different phases (order = phase offsets)
% Each packet: short radial tick + label; use varied phases
\foreach \ang/\lbl in {20/$g_1$, 60/$g_2$, 140/$g_3$, 200/$g_4$, 300/$g_5$}{
  \draw[gpack, blue!70] ({\cx+0.82*\R*cos(\ang)},{\cy+0.82*\R*sin(\ang)}) --
                        ({\cx+0.98*\R*cos(\ang)},{\cy+0.98*\R*sin(\ang)});
  \node[blue!75] at ({\cx+1.15*\R*cos(\ang)},{\cy+1.15*\R*sin(\ang)}) {\scriptsize \lbl};
}

% Closure annotation (cycle reset)
\draw[->,gray!65, line width=0.9pt] ({\cx+\R*0.1},{\cy-\R*1.18}) .. controls (\cx+3.8,-2.0) .. ({\cx+\R*0.1},{\cy+\R*1.18});
\node[gray!65] at (\cx+3.9,-0.1) {\scriptsize closure / reset ($\theta$ cycles)};

% Legends / braces
\draw[decorate, decoration={brace, amplitude=5pt}, gray!60] (-5.6,-0.8) -- (-2.2,-0.8)
  node[midway, below=6pt] {\scriptsize linear time};
\draw[decorate, decoration={brace, amplitude=5pt}, gray!60] (\cx-\R-0.6,-1.9) -- (\cx+\R+0.6,-1.9)
  node[midway, below=6pt] {\scriptsize cyclic parameterization of events on $S^1$};

% Explanatory labels
\node[align=left] at (-1.2,1.6) {\footnotesize \textbf{Binding:} events in the same};
\node[align=left] at (-1.2,1.2) {\footnotesize \textbf{phase window} group together};
\node[align=left] at (-0.6,0.35) {\footnotesize \textbf{Ordering:} sequence $\;\sim$ phase offsets};
\node[align=left] at (-0.6,-0.55) {\footnotesize \textbf{Closure:} cycle resets $\Rightarrow~\partial^2\!=\!0$};
\end{tikzpicture}
\caption{\textbf{Oscillatory phase coding as topological closure.}
Linear time is wrapped onto the circle $S^1$ (theta phase), placing events by \emph{phase}
rather than absolute time. Gamma packets ($g_1,\dots,g_5$) at distinct phases encode \emph{order}
via phase offsets, while events within the same phase window (green arc) are \emph{bound}.
Cycle reset at the end of each theta period enforces topological closure ($\partial^2=0$),
supporting consistent memory cycles.}
\label{fig:phase-coding-closure}
\end{figure}

\subsection{Coincidence Detection as Topological Closure}

\noindent Lemma~\ref{lem:phase-closure} establishes that oscillatory phase 
coding furnishes a natural scaffold for aligning events on the circle $S^1$, 
ensuring that candidate trajectories can be organized into cyclic frames. 
However, phase alignment alone does not guarantee stability: without a 
mechanism to prune misaligned or inconsistent events, spurious boundaries 
would accumulate and prevent reliable cycle formation. 
Lemma~\ref{lem:coincidence-closure} addresses this gap by showing how 
coincidence detection enforces closure at the level of spike trains, 
cancelling mismatches as boundary terms and preserving only those 
cycles that survive across trials. Together, these results formalize 
the complementary roles of phase scaffolding and coincidence pruning in 
transforming transient alignments into reproducible cognitive invariants.

\begin{lemma}[Coincidence-Induced Closure and Survival of Reproducible Cycles]
\label{lem:coincidence-closure}
Let $\mathcal{N}=\{1,\dots,n\}$ be a set of units (neurons) producing a spike
train $S=\{(i,t_k)\}$ with phases $\phi(t_k)\in S^1$. Fix a coincidence window
$\Delta\in(0,\pi)$ and define the \emph{coincidence relation}
$i \stackrel{\Delta}{\leftrightarrow} j$ iff there exist spikes
$(i,t),(j,t')\in S$ with $|\phi(t)-\phi(t')|_{S^1}\le \Delta$ and
$t<t'$ (to orient time).
Construct the directed 1\mbox{--}skeleton $G_\Delta(S)$ whose vertex set is
$\mathcal{N}$ and whose (possibly multiple) oriented edges are
$e=(i\!\to\! j)$ for every coincident pair $i \stackrel{\Delta}{\leftrightarrow} j$.
Let $C_1(G_\Delta)$ be the free abelian group on edges and $C_0(G_\Delta)$ the
free abelian group on vertices, with boundary $\partial:C_1\to C_0$ given by
$\partial(i\!\to\! j)=j-i$.
Define the \emph{coincidence aggregation} $c_\Delta(S)\in C_1(G_\Delta)$ by
summing all oriented edges (with multiplicities) generated by coincident pairs,
and the \emph{coincidence projection}
$\Pi_\Delta: C_1(G_\Delta)\longrightarrow Z_1(G_\Delta):=\ker\partial$
as the (linear) projection onto the cycle space (e.g., orthogonal projection
with respect to any inner product on $C_\bullet$ or the canonical decomposition
$C_1=Z_1\oplus B_1^\perp$).
Then we have:
1) \textbf{Closure by coincidence.} The \emph{coincidence detector}
$K_\Delta:=\Pi_\Delta\circ(\cdot)$ enforces closure:
$z_\Delta(S):=K_\Delta\big(c_\Delta(S)\big)\in Z_1(G_\Delta)~\text{and}~
\partial\, z_\Delta(S)=0$.
Moreover, the edges removed by $K_\Delta$ are precisely those whose net
contribution appears in $\partial c_\Delta(S)$; i.e.\ misaligned spikes are
canceled as boundary terms and do not survive in $z_\Delta(S)$.
2) \textbf{Survival of reproducible cycles (stability).}
Suppose $S^{(1)},\dots,S^{(T)}$ are trials with phase jitter at most
$\varepsilon<\Delta$ (i.e.\ every coincidence in one trial has a matched
coincidence within phase distance $\varepsilon$ in all others, with the same
orientation). Then for all $t$,
$[z_\Delta(S^{(t)})]\;=\;[z_\Delta(S^{(1)})]\in H_1(G_\Delta;\mathbb{Z})$,
so the homology class is trial\mbox{-}invariant. In particular, in the
persistence module obtained by varying the window $\delta\in(\varepsilon,\Delta]$,
this class has positive lifetime and therefore \emph{survives} while
nonreproducible coincidences die as boundaries.
\end{lemma}

\begin{proof}[Proof sketch]
(1) By construction, $\partial c_\Delta(S)$ counts net imbalance of incident
coincidences at each vertex (incoming minus outgoing). Projecting onto
$\ker\partial$ removes exactly those components whose boundary is nonzero; hence
$z_\Delta(S)\in Z_1$ and $\partial z_\Delta(S)=0$. Informally, coincidences that
do not close are eliminated as boundary terms; only closed flow persists.
(2) Phase jitter $\varepsilon<\Delta$ induces edge correspondences between the
$G_\Delta(S^{(t)})$ that preserve orientation and incidence, yielding chain
homotopic $c_\Delta(S^{(t)})$. Projection to $Z_1$ commutes with these
homotopies, so the resulting $z_\Delta(S^{(t)})$ are homologous. Viewing
$\delta$ as a filtration parameter, unmatched (nonreproducible) edges vanish at
$\delta\searrow\varepsilon$, whereas reproducible cycles define a bar of
positive length in $H_1$, hence survive.
\end{proof}

Lemma~\ref{lem:coincidence-closure} established that coincidence detection enforces closure 
by cancelling misaligned spikes, ensuring that only reproducible cycles survive. 
Fig.~\ref{fig:coincidence-boundary-cancellation} illustrates this principle: 
when presynaptic spikes align within a coincidence window $\Delta$ (top), their 
inputs sum coherently and trigger a postsynaptic spike, corresponding to 
$\partial\gamma=0$ (closure). When spikes fall outside the window (bottom), they 
remain as unmatched boundaries that cancel one another, yielding no output. The 
inset shows the topological analogy: different paths that bound the same face 
$\sigma$ cancel in homology, just as misaligned temporal fragments fail to 
stabilize into persistent cycles. Formally, we have 

\begin{definition}[Topological Closure]
Let $(X,\tau)$ be a topological space and $A \subseteq X$. 
The \emph{closure} of $A$, denoted $\overline{A}$, is defined as
$\overline{A} = \bigcap \{ C \subseteq X \mid C \text{ is closed and } A \subseteq C \}$.
Equivalently, $\overline{A}$ consists of all points $x \in X$ such that every
open neighborhood $U \in \tau$ with $x \in U$ satisfies $U \cap A \neq \varnothing$.
\end{definition}

\noindent With the formal notion of closure in hand, we now \emph{operationalize}
it in neural dynamics: replace open neighborhoods by temporal coincidence windows
and subsets $A$ by sets of candidate spikes. Under this identification, the
“points in the closure” are precisely spikes that recurrently co-occur within a
window, and the homological reading of closure ($\partial^2=0$) corresponds to
cancelling unmatched, out-of-window events as boundary terms. This yields a
direct bridge from topological closure to coincidence-driven cycle formation in
neural circuits.
%\paragraph{Coincidence detection as topological closure ($\Phi$).}
For a PNG to persist, spikes from multiple presynaptic neurons must converge 
within a narrow temporal window at their postsynaptic targets. Coincidence 
detection acts as a filter: inputs that arrive in synchrony are integrated, 
while those that fall outside the coincidence window are effectively cancelled. 
This selective integration implements the algebraic identity $\partial^2=0$: 
misaligned spikes behave like open boundaries that fail to connect, whereas synchronous arrivals cancel boundary terms and enforce cycle closure. In this way, only temporally coherent activity contributes to a closed 1-cycle in the neural state space. Once closure is achieved, spike-timing dependent plasticity (STDP) reinforces 
the recurrent pathways that produced coincident input \cite{caporale2008spike}. Potentiation strengthens 
the synapses along routes that consistently deliver spikes within the window $\Delta$, while depression weakens those that fail to align. Over repeated 
activations, this differential plasticity stabilizes the trajectory as a 
reentrant cycle: the cycle not only replays reliably, but also becomes resistant 
to perturbations of individual spike times. In summary, coincidence detection, 
together with STDP, extracts the low-entropy content variable $\Phi$: a 
reproducible invariant that persists as a memory trace and can later be 
recalled or recombined into higher-order structures \cite{li2025memory}.

The principle ``coincidence detection = boundary cancellation'' can now be 
made explicit. When presynaptic spikes converge within the coincidence window, 
their temporal boundaries align and cancel, producing a closed cycle that can 
drive a stable postsynaptic response \cite{konig1996integrator}. In contrast, when spikes arrive outside 
the window, they leave residual unmatched boundaries that fail to close, and 
no postsynaptic output is generated. Figure~\ref{fig:coincidence-boundary-cancellation} 
illustrates this correspondence: in the neural case, misaligned spikes cancel 
each other’s contributions and disappear from the effective cycle; in the 
topological case, paths that differ by the boundary of a 2-simplex $\sigma$ 
cancel in homology. In both settings, coincidence detection enforces the 
identity $\partial^2=0$, ensuring that only closed cycles survive as 
memory-bearing invariants.

% ---------- TikZ: Coincidence Detection = Boundary Cancellation ----------
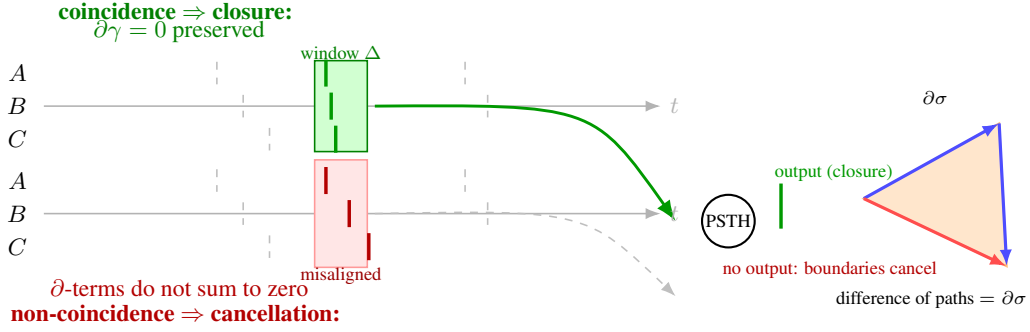
\begin{figure}[t]
\centering
\begin{tikzpicture}[
  >=Latex,
  scale=1.0,
  every node/.style={font=\small},
  axis/.style={gray!60, line width=0.6pt, -{Latex[length=2mm]}},
  spike/.style={line width=1.3pt},
  coinc/.style={green!60!black, line width=1.3pt},
  noco/.style={red!70!black, line width=1.3pt},
  win/.style={fill=green!18, draw=green!50!black, line width=0.6pt},
  guide/.style={gray!50, line width=0.6pt, dashed}
]

% Layout helpers
\def\L{8.2}     % time axis length
\def\H{1.1}     % row height
\def\x0{-0.2}   % left margin
\def\cx{8.9}    % neuron x
\def\cy{ 0.0}   % neuron y

% ========== TOP ROW: Coincident inputs -> Output spike (closure) ==========
% Axes
\draw[axis] (\x0, 1.8*\H) -- (\x0+\L, 1.8*\H) node[right] {$t$};
% Labels presynaptic
\node[left] at (\x0-0.1, 2.2*\H) {$A$};
\node[left] at (\x0-0.1, 1.8*\H) {$B$};
\node[left] at (\x0-0.1, 1.4*\H) {$C$};

% Coincidence window
\def\wL{3.4} \def\wR{4.1}
\filldraw[win] (\wL, 1.25*\H) rectangle (\wR, 2.35*\H);
\node[green!50!black] at ({0.5*(\wL+\wR)}, 2.45*\H) {\scriptsize window $\Delta$};

% Spikes aligned within window
\foreach \y/\x in {2.2*\H/3.55, 1.8*\H/3.62, 1.4*\H/3.68}{
  \draw[coinc] (\x, \y+0.18) -- (\x, \y-0.18);
}
% A few extra earlier/later spikes (gray)
\foreach \y/\x in {2.2*\H/2.1, 1.8*\H/2.45, 1.4*\H/2.8, 2.2*\H/5.4, 1.8*\H/5.7}{
  \draw[guide] (\x, \y+0.15) -- (\x, \y-0.15);
}

% Arrow from window to postsynaptic neuron + output spike
\draw[->, green!60!black, line width=1.1pt] (\wR+0.1, 1.8*\H) .. controls (7.1,2.0) .. (8.2,0.45);
% Postsynaptic neuron
\draw[thick] (\cx,0.45) circle (0.35);
\node at (\cx,0.45) {\scriptsize PSTH};
% Output spike (top row)
\draw[coinc] (\cx+0.7,0.95) -- (\cx+0.7,0.35);
\node[green!60!black] at (\cx+1.4,1.1) {\scriptsize output (closure)};

% Annotation
\node[align=left, green!50!black] at (1.6, 2.95*\H) {\textbf{coincidence $\Rightarrow$ closure:} };
\node[align=left, green!50!black] at (1.6, 2.7*\H) {$\partial \gamma = 0$ preserved};

% ========== BOTTOM ROW: Non-coincident inputs -> No output (boundaries cancel) ==========
% Axes
\draw[axis] (\x0, 0.5*\H) -- (\x0+\L, 0.5*\H) node[right] {$t$};
% Labels presynaptic
\node[left] at (\x0-0.1, 0.9*\H) {$A$};
\node[left] at (\x0-0.1, 0.5*\H) {$B$};
\node[left] at (\x0-0.1, 0.1*\H) {$C$};

% Misaligned "window" visual (light red band)
\def\mL{3.4} \def\mR{4.1}
\filldraw[fill=red!10, draw=red!40, line width=0.6pt] (\mL, -0.15*\H) rectangle (\mR, 1.15*\H);
\node[red!60!black] at ({0.5*(\mL+\mR)}, -.25*\H) {\scriptsize misaligned};

% Spikes scattered (not overlapping in window)
\foreach \y/\x in {0.9*\H/3.55, 0.5*\H/3.86, 0.1*\H/4.12}{
  \draw[noco] (\x, \y+0.18) -- (\x, \y-0.18);
}
% Extras
\foreach \y/\x in {0.9*\H/2.1, 0.5*\H/2.45, 0.1*\H/2.8, 0.9*\H/5.4, 0.5*\H/5.7}{
  \draw[guide] (\x, \y+0.15) -- (\x, \y-0.15);
}

% Arrow to neuron (dashed, weak) and no output
\draw[guide, ->] (\mR+0.1, 0.5*\H) .. controls (7.0,0.6) .. (8.2,-0.55);
% Postsynaptic neuron (same circle reused visually at y=-0.55 not to crowd; we will annotate instead)
\node[red!70!black] at (\cx+1.35,-0.2) {\scriptsize no output: boundaries cancel};

% Annotation
\node[align=left, red!70!black] at (1.6, -0.7*\H) {\textbf{non-coincidence $\Rightarrow$ cancellation:} };
\node[align=left, red!70!black] at (1.6, -0.4*\H) {$\partial$-terms do not sum to zero};

% ========== RIGHT INSET: Topological analogy (boundary of a face) ==========
\begin{scope}[shift={(10.7,0.75)}]
  % Triangle face
  \fill[orange!20] (0,0) -- (1.8,1.0) -- (1.9,-0.9) -- cycle;
  \draw[orange!50, line width=0.8pt] (0,0) -- (1.8,1.0) -- (1.9,-0.9) -- cycle;
  % Two paths, opposite orientations (difference = boundary)
  \draw[blue!70, line width=1.2pt, -{Latex[length=2.2mm]}] (0,0) -- (1.8,1.0);
  \draw[blue!70, line width=1.2pt, -{Latex[length=2.2mm]}] (1.8,1.0) -- (1.9,-0.9);
  \draw[red!70,  line width=1.2pt, -{Latex[length=2.2mm]}] (0,0) -- (1.9,-0.9);
  \node at (0.95,1.35) {\scriptsize $\partial\sigma$};
  \node at (0.9,-1.35) {\scriptsize difference of paths $=\partial\sigma$};
\end{scope}

\end{tikzpicture}
\caption{\textbf{Coincidence detection = boundary cancellation.}
\emph{Left:} Three presynaptic spike trains ($A,B,C$). \textbf{Top:} Spikes align within a coincidence window $\Delta$
(green band), summate, and produce a postsynaptic spike (closure). \textbf{Bottom:} Spikes are misaligned; inputs fail to
coincide in $\Delta$, so no output occurs (boundaries do not cancel). \emph{Right inset:} Topological analogy - when two
paths differ by the boundary of a face ($\partial\sigma$), their difference cancels in homology; likewise, misaligned
temporal fragments behave as open boundaries, while coincidence implements $\partial^2=0$, leaving only closed cycles.}
\label{fig:coincidence-boundary-cancellation}
\end{figure}

\begin{principle}[Coincidence Detection as Boundary Cancellation]
Let $\{v_i\}$ denote neural events (e.g., spikes) indexed in time, and let 
$e_{ij} = [v_i,v_j]$ denote a directed edge formed when two events fall 
within a coincidence window $\Delta t$. Define a chain 
$c = \sum_{(i,j)} e_{ij}$ over all coincident pairs. 
1) If events are misaligned ($|t_i - t_j| > \Delta t$), 
    no edge is formed; the fragment remains an open chain with 
    nonvanishing boundary $\partial e_{ij} = v_j - v_i$.
2) If events are coincident ($|t_i - t_j| \leq \Delta t$), 
    opposite boundaries cancel: $\partial c = \sum (v_j - v_i) = 0$.
In summary, coincidence detection implements the algebraic rule 
$\partial^2=0$: unmatched endpoints dissipate, while synchronous 
inputs enforce closure. Biologically, this ensures that only 
temporally aligned inputs reinforce into stable cycles, whereas 
misaligned fragments are pruned. 
\end{principle}

\begin{example}[Grid/Place Cells via Oscillatory Phase Coding and Coincidence Detection]
\label{ex:grid-place-osc-coin}
Let an animal move along a path $x(t)\in\mathbb{R}^2$ during exploration. 
A hippocampo-entorhinal \emph{theta} oscillation provides a cyclic temporal scaffold 
$\theta(t)=\omega_\Theta t \;\mathrm{mod}\;2\pi$ on $S^1_\Theta$ (\textbf{oscillatory phase coding}). 
Entorhinal grid cells $g_j$ are modeled by spatial phases 
$\varphi_j(x)\;=\;\langle k_j,\,x\rangle+\phi_j\;\;\mathrm{mod}\;2\pi$,
with wavevectors $k_j$ arranged on (approximately) hexagonal lattices; jointly these induce a torus code on $T^m=(S^1)^m$ and, when combined with theta, a phase space 
$S^1_\Theta\times T^m$.
\noindent\textbf{Coincidence detection.}
Let a downstream (hippocampal) \emph{place cell} $p$ pool inputs from a set $\mathcal{G}$ of grid cells 
with synaptic weights $w_j>0$. Define a coincidence kernel $\kappa_\Delta$ (e.g., von Mises or boxcar) 
of width $\Delta\!\ll\!2\pi$ on the circle. The phase-locked input to $p$ at time $t$ is
$I_p(t)\;=\;\sum_{j\in\mathcal{G}} w_j\;\kappa_\Delta\!\Big(\theta(t)-\varphi_j\big(x(t)\big)\Big)$.
The postsynaptic cell fires when the \emph{theta-gated coincidence functional}
$\mathcal{C}_p[x]\;=\;\frac{1}{T_\Theta}\int_{t_0}^{t_0+T_\Theta} I_p(t)\,dt$
exceeds a threshold $\tau$; here $T_\Theta=2\pi/\omega_\Theta$. Because $\kappa_\Delta$ is sharply peaked, 
misaligned phases average out over a cycle, while aligned phases \emph{sum coherently}.
\noindent\textbf{Claims.}
1) (\emph{Scaffold $\Psi$}) Theta imposes a cyclic temporal parameterization $S^1_\Theta$, and grid phases 
$\varphi_j$ tile space periodically; together they provide a high-entropy scaffold 
$\Psi$ over $S^1_\Theta\times T^m$.
2) (\emph{Content $\Phi$ via coincidence}) A \emph{place field} for $p$ emerges at locations $x^\ast$ where 
\emph{multi-lattice phase alignment} holds: $\theta(t^\ast)\approx \varphi_j(x^\ast)$ (mod $2\pi$) for many $j$.
At such $x^\ast$, $\mathcal{C}_p[x]>\tau$ and $p$ fires reproducibly, defining a low-entropy content invariant $\Phi$.
3) (\emph{Boundary cancellation $\Rightarrow$ closure}) Along an exploratory trajectory segment $\sigma$ that 
does \emph{not} return near $x^\ast$, phase misalignment causes $\int_\sigma I_p(t)\,dt$ to cancel (\emph{boundary terms} 
do not accumulate). When the animal revisits $x^\ast$ in a closed tour $\gamma$ (home cycle), aligned contributions 
add over $\theta$-cycles while misaligned ones cancel, yielding a \emph{closed} neural trajectory with 
$\partial\gamma=0$. Thus coincidence detection implements \emph{boundary cancellation} (algebraically $\partial^2=0$), 
so only closed cycles persist.
4) (\emph{Order invariance}) If two paths $\gamma, \eta$ traverse the same set of grid phases and both return 
to home, the cycle-integrated coincidence depends on the \emph{multiset} of aligned phase hits, not their order. 
Hence the induced class in $H_1$ is identical: $[\gamma]=[\eta]$.
5) (\emph{Homological substrate}) With two incommensurate grid axes plus theta, phases live on 
$S^1_\Theta\times S^1\times S^1$, whose first homology is $H_1\cong\mathbb{Z}^3$. Integer windings encode 
\emph{hierarchical cycles}: gamma-in-theta microcycles accumulate into spatial reentrant cycles that stabilize place fields.
\noindent\textbf{Conclusion.}
Oscillatory phase coding supplies the \emph{scaffold} ($\Psi$), and coincidence detection prunes it to 
\emph{content} ($\Phi$) by canceling misaligned inputs and reinforcing aligned ones. The surviving, reproducible 
grid-place patterns are \emph{closed cycles} that are insensitive to traversal order and persist as homology classes. 
Thus, grid and place cells instantiate \emph{topological closure}: memory arises as invariant cycles formed by 
theta-gated phase alignment and coincidence-based boundary cancellation.
\end{example}

\paragraph{Summary.}
Oscillatory phase coding supplies the cyclic scaffold ($\Psi$), ensuring that 
experience is embedded within recurrent temporal frames, while PNGs implement 
content cycles ($\Phi$) stabilized by coincidence detection. Together they 
instantiate the algebraic identity $\partial^2=0$ in neural dynamics \cite{gerstner2014neuronal}: 
oscillations guarantee that trajectories are organized into cycles, and PNGs 
ensure that only nontrivial, reproducible cycles survive as memory traces. 
In this way, $\Psi$ and $\Phi$ are dynamically aligned: the scaffold partitions 
time into contextual windows, while coincidence detection filters and stabilizes 
only those spike patterns that achieve closure within these windows. The outcome 
is a neural substrate in which memory is not stored as isolated events, but as 
persistent cycles embedded in oscillatory structure \cite{buzsaki2006rhythms}.

This perspective naturally motivates a closer examination of polychronous neural 
groups (PNGs) \cite{izhikevich2006polychronization}. If oscillations provide the scaffold for cycle formation, then PNGs 
are the concrete realizations of those cycles in spiking networks. By combining 
axonal delays, coincidence detection, and STDP, PNGs produce reproducible 
spatiotemporal trajectories that can be iterated, nested, and recombined. 
While oscillatory phase coding establishes the global conditions for 
closure, PNGs reveal how the brain implements cycle formation locally. In the 
next subsection, we formalize this principle by showing how PNGs organize 
cognition as a \emph{hierarchy of cycles}, ranging from micro-level 1-cycles to 
higher-order perceptual, motor, and abstract cycles.

\subsection{PNGs as Hierarchical Cycle Machine}

\noindent
\textbf{Biological picture.}
Polychronous neural groups (PNGs) provide a natural substrate for understanding 
the emergence of invariant cycles in biological neural networks 
\cite{izhikevich2006polychronization}. Their defining property is that precise 
spike-timing patterns, stabilized by axonal delays and spike-timing dependent 
plasticity (STDP) \cite{caporale2008spike}, generate reproducible closed cycles 
in neural state space. In this subsection, we formalize the principle that PNGs 
organize cognition as a \emph{hierarchy of cycles}, ranging from micro-level 
1-cycles to higher-order perceptual, motor, and abstract cycles (see 
Fig.~\ref{fig:png-hierarchy} for a schematic illustration).
At the mechanistic level, axons with heterogeneous delays $d_{ij}$ and STDP 
create polychronous matches: spikes arriving within a coincidence window 
$\Delta$ ignite a reproducible spatiotemporal pattern, forming a PNG. When the 
sum of delays around a route resonates with a carrier rhythm (e.g., $\theta$), 
recurrent activation \emph{closes a cycle}. Importantly, different micro-routes 
that arrive within $\Delta$ collapse to the same postsynaptic event, yielding 
\emph{order-invariant} readout. This shows that PNGs naturally implement the 
principle of closure: only spike configurations that align in time survive as 
stable cycles, while misaligned ones vanish as boundaries. Moreover, 
cross-frequency nesting (\emph{gamma-in-theta}) and replay (ripples) compose 
these micro-cycles into a hierarchy that amortizes past solutions, illustrating 
how memory consolidation builds upon the same cycle-forming machinery.

\paragraph{PNGs as 1-cycles ($\Phi$).}
The biological description can be formalized topologically. Within the 
oscillatory scaffold ($\Psi$), PNGs are embedded as the concrete realizations of 
closed 1-cycles. Each PNG corresponds to a delay-locked trajectory of spikes 
$\gamma$ that is reproducible across trials and stable under perturbations of 
individual spike times. Coincidence detection enforces $\partial \gamma = 0$: 
open spike chains cancel at their boundaries, leaving only closed cycles. 
Therefore, PNGs correspond to nontrivial homology classes 
$[\gamma]\in H_1(\mathcal{Z})$, which serve as the low-entropy content units 
($\Phi$) of memory. This establishes PNGs as the bridge between biological 
mechanisms of spike timing and the algebraic invariants of homology.
Having shown how PNGs realize 1-cycles, we now extend the construction beyond 
the micro-scale. The key insight is that the same closure and order-invariance 
principles can be iterated and composed: {\em cycles can be coupled across modalities 
(perception and action), and further composed into higher-order cycles that 
capture abstract plans and concepts}. To connect this insight more 
directly to spike-level mechanisms, we next characterize how polychronous closure produces neural 1-cycles with order-invariant readout.

\begin{theorem}[Polychronous Closure $\Rightarrow$ Neural 1-Cycles with Order-Invariant Readout]
\label{thm:png_closure}
Let $\mathcal{G}=(V,E)$ be a directed synaptic graph with delays $d_{ij}>0$ and weights $w_{ij}$ updated by STDP.
Fix a coincidence window $\Delta>0$ and a carrier period $T_\Theta$ (e.g., $\theta$).
Suppose there exists a directed cycle $\gamma=(i_1\!\to i_2\!\to \cdots \!\to i_m\!\to i_1)$ such that
$\Big|\textstyle\sum_{k=1}^{m} d_{i_k i_{k+1}} - n T_\Theta\Big|\ \le \ \Delta
\quad\text{for some } n\in\mathbb{Z},
\quad\text{and}\quad \prod_{k=1}^{m} w_{i_k i_{k+1}} \ \ge\ \tau$,
for a cycle-gain threshold $\tau>0$. Then we have:
1) (\emph{reentrant closure}) The induced Poincar\'e map on the phase torus admits a stable fixed point:
the PNG reactivates after $nT_\Theta$ (within $\Delta$) and forms a recurrent \emph{neural 1-cycle}.
2) (\emph{order-invariant readout}) Any micro-variation of spike routes that preserves coincidence within $\Delta$
produces the same postsynaptic event and the same closed cycle class in the timing-aware chain complex (i.e., the readout is invariant to the order of converging micro-paths).
\end{theorem}

\begin{proof}[Proof sketch]
STDP aligns effective delays along potentiated paths; resonance with $T_\Theta$ yields a return map near the identity on phase. cycle gain above threshold stabilizes recurrence (1). Coincidence detection integrates arrivals within $\Delta$, so permutations of converging micro-routes are indistinguishable at readout; algebraically they sum to the same $1$-chain with $\partial z=0$ (2). 
\end{proof}

The following corollary generalizes the Hebbian learning principle from the cycle persistence perspective. 

\begin{corollary}[Hebbian-STDP Reinforcement $\Rightarrow$ Polychronous Cycle Persistence]
\label{cor:hebb-stdp-persistence}
Consider the setting of Lemma~\ref{lem:coincidence-closure}. Suppose a set of synapses 
along a directed cycle $\gamma=(i_1\!\to\cdots\!\to i_m\!\to i_1)$ undergo repeated 
co-activation such that:
1) (\emph{Temporal Hebbian condition}) Pre-before-post spike timing on each edge 
$(i_k\!\to i_{k+1})$ falls within the potentiating STDP window on average, while 
post-before-pre occurrences fall predominantly in the depressing window off the cycle;
2) (\emph{Coincidence condition}) The cycle’s effective delay sums satisfy 
$\big|\sum_{k} d_{i_k i_{k+1}} - nT_\Theta\big|\le \Delta$ for some $n\in\mathbb{Z}$, 
with coincident arrivals at postsynaptic targets within the window $\Delta$;
3) (\emph{Gain condition}) The induced weight dynamics drive 
$\prod_{k} w_{i_k i_{k+1}} \ge \tau$ (cycle-gain threshold).
Then the cycle $\gamma$ becomes a persistent \emph{neural 1-cycle} that reliably 
reenters after $nT_\Theta$ and yields order-invariant readout. In particular, the 
closed-cycle memory trace corresponds to a nontrivial class $[\gamma]\in H_1(\mathcal{Z})$.
\end{corollary}

\begin{proof}[Proof sketch]
Condition (1) implements a temporally refined Hebbian rule (STDP) that selectively 
potentiates synapses consistent with the cycle’s pre$\!\to\!$post timing, while 
suppressing competing routes. Condition (2) ensures coincidence detection: misaligned 
spikes cancel as boundaries, whereas aligned arrivals enforce $\partial \gamma=0$. 
Condition (3) guarantees that the reentrant trajectory is self-sustaining; together 
with the $\theta$ resonance in (2), the induced Poincar\'e map admits a stable fixed 
point. Theorem~\ref{thm:png_closure} then applies, establishing both persistence and 
order-invariant readout; nontriviality in homology follows from closure under 
perturbations and the failure to contract to a boundary. %\qedhere
\end{proof}

\paragraph{Theta-gamma nesting as scaffold-closure hierarchy.}
From this perspective, slow rhythms such as theta ($4$-$8$ Hz) furnish the 
contextual scaffold ($\Psi$) by segmenting experience into macro-cycles, while 
faster gamma oscillations ($30$-$100$ Hz) embed discrete content packets 
within these cycles. Oscillatory phase coding provides the temporal scaffold: 
linear time is folded onto the circle $S^1$, so that events are organized by 
phase rather than absolute chronology \cite{LismanJensen2013,CanoltyKnight2010,Canolty2006PNAS,BuzsakiDraguhn2004}. 
This cyclic parameterization ensures that each theta cycle acts as a contextual 
frame, within which gamma bursts encode ordered but bounded packets of content. 

Coincidence detection then enforces closure inside this scaffold: spikes from 
multiple presynaptic sources must converge within narrow gamma windows to 
generate a postsynaptic response \cite{KonigEngelSinger1996,StuartSakmann1994,BuzsakiWang2012}. 
Misaligned spikes cancel as open boundaries, while coincident spikes close into 
invariant cycles, implementing the homological identity $\partial^2=0$. 
STDP further reinforces those gamma packets that reliably fall into the same 
theta-defined temporal slots, pruning transient alignments, and stabilizing 
reproducible cycles \cite{Markram1997Science,BiPoo1998,CaporaleDan2008}. 
The result is a nested scaffold-closure hierarchy: theta oscillations create 
high-entropy contextual windows ($\Psi$), gamma packets embed low-entropy content 
($\Phi$) within them, and coincidence detection ensures that only closed and 
reproducible gamma cycles persist. In this way, theta-gamma nesting realizes the 
dynamic alignment of $\Psi$ and $\Phi$ biologically, yielding coherent cycles that 
can be recalled, revised, or recombined across scales of perception, action, and 
memory \cite{Fries2015}, as illustrated in Fig.~\ref{fig:png-hierarchy}.

\begin{proposition}[Cross-Frequency Nesting and Replay Induce a Hierarchy of Cycles]
\label{prop:png-hierarchy}
Let $\Gamma_\gamma$ denote the family of gamma PNG cycles detected within $\theta$ phase bins
$\{\varphi_\ell\}_{\ell=1}^L$, and let $\mathcal{R}$ denote ripple-triggered replay during NREM/quiet wake.
Assume: (i) $\theta$ phase gates PNG ignition; (ii) replay compresses and replays subsets of $\Gamma_\gamma$
and induces weight updates along the same cycles. Then:
1) (\emph{Nested cycles}) Phase-indexed gamma cycles assemble into a $\theta$-indexed ring, yielding a
closed trajectory in $\mathcal{Z}_\gamma \times S^1_\theta$ (a higher-order cycle).
2) (\emph{Consolidation}) Replay $\mathcal{R}$ increases cycle gains selectively on previously closed PNG cycles,
pruning trivial cycles and stabilizing nontrivial homology classes (persistent memories).
3) (\emph{Perception-action lift}) Cross-areal reentry (cortico-BG-thalamo-hippocampal) composes perceptual PNG
cycles with motor PNG cycles, producing joint cycles in $\mathcal{Z}_P \times \mathcal{Z}_A$ that are invariant to
micro-order and cross-modal interleavings.
\end{proposition}

Proposition~\ref{prop:png-hierarchy} describes how cross-frequency nesting and 
replay compose local PNG cycles into a multiscale hierarchy of cycles \cite{theta_gamma_cfc_memory}. However, this 
hierarchy does not simply accumulate structure; it must also undergo daily 
selection, pruning, and stabilization. The sleep-wake rhythm provides this 
mechanism \cite{fuller2006neurobiology}: waking generates open exploratory chains, while sleep implements 
closure through replay and reactivation. In this way, the global sleep-wake 
cycle can be understood as the biological realization of the topological 
identity $\partial^2 = 0$ \cite{buzsaki1996hippocampo}: boundaries left unresolved during wake are either 
canceled or closed during sleep, ensuring that only nontrivial cycles persist as 
consolidated memories. This intuition is formalized in the following lemma.

\begin{example}[Hierarchy of Perception-Action, Meso-Scale, and Sleep-Wake Cycles]
\label{ex:cycle-hierarchy}
The perception-action cycle (seconds to minutes) and the sleep-wake cycle
(hours to days) are not isolated phenomena but instances of a nested cycle 
hierarchy induced by symmetry breaking. In between these extremes lie 
\emph{meso-scale cycles} (minutes to hours), which mediate between 
short-term sensorimotor alignment and long-term memory consolidation. 
Examples include working-memory refresh cycles, decision-making cycles,
theta-gamma nesting during episodic encoding, and ultradian rhythms 
such as the 90-minute Non-Rapid Eye Movement(NREM)-Rapid Eye Movement(REM) alternation. Across these scales:
1) \emph{Micro-scale (perception-action, seconds to minutes).} 
    Local orderings of events (saccades, motor primitives) are irrelevant once 
    the cycle closes, yielding order-invariant cycles that stabilize 
    sensorimotor alignment and immediate goal achievement \cite{fuster2004upper}.
2) \emph{Meso-scale (minutes to hours).}  
    Working memory refresh cycles and theta-gamma nested cycles serve as 
    intermediate scaffolds, amortizing micro-scale sensorimotor traces 
    into episodic representations. Ultradian rhythms (e.g., NREM-REM alternations \cite{maquet1996functional}) 
    reorganize memory traces within a waking day, bridging online activity 
    with nightly consolidation. These cycles propagate information upward, 
    ensuring that micro traces do not dissipate but are integrated into 
    broader episodic contexts.
3) \emph{Macro-scale (sleep-wake, hours to days).}  
    Daily alternations between waking experience and offline replay integrate 
    meso-scale traces into long-term memory \cite{fuller2006neurobiology}. Wake generates exploratory, 
    open trajectories ($\sigma\in C_1$), while sleep provides closure via 
    coincidence detection (hippocampal replay, ripple nesting), pruning 
    trivial paths and reinforcing nontrivial homology classes $[\gamma]\in H_1$.
4) \emph{Persistence across scales.}  
    By $\partial^2=0$, boundaries of boundaries cannot accumulate, 
    ensuring that only closed cycles persist. Micro-cycles feed into meso-cycles, 
    which in turn are consolidated at the macro scale, forming a recursive 
    bootstrapping mechanism \cite{li2025information}. This hierarchical organization guarantees that 
    memory traces remain invariant under reordering and perturbations, 
    while scaling up from immediate perception--action interactions to 
    circadian-long integration.
\emph{Summary.} Micro, meso, and macro cycles are not separate mechanisms 
but nested levels of a homological hierarchy. Each level amortizes and 
stabilizes the invariants of the scale below it, while contributing 
structured inputs to the scale above. Together they illustrate the 
bootstrapping principle of cognition: cycles at one timescale seed, 
reinforce, and reorganize cycles at another, ensuring coherence of 
memory and prediction across time.
\end{example}

The above example establishes that perception-action cycles \cite{fuster2004upper} and sleep-wake cycles \cite{fuller2006neurobiology}
are not separate mechanisms but nested layers of the same homological hierarchy, 
distinguished only by their temporal scale. The next step is to see how this 
principle is realized concretely in the daily alternation between waking 
experience and sleep-dependent replay. In particular, sleep provides the 
mechanism by which open chains of waking experience are aligned, pruned, and 
closed into stable cycles, ensuring that only nontrivial invariants persist as 
memory. The following corollary formalizes this role of the sleep-wake cycle 
as an operator of topological closure.

\begin{corollary}[Sleep-Wake Cycle as Topological Closure]
\label{cor:sleep-wake-closure}
Let $\mathcal{Z}$ denote the neural state space, with trajectories generated 
during wake and replayed during sleep. Then the sleep--wake cycle enforces 
the homological identity $\partial^2 = 0$ at the daily scale:
1) (\emph{Wake as open chains}) Waking experience generates exploratory 
trajectories $\sigma \in C_1(\mathcal{Z})$ that need not close, corresponding 
to high-entropy contextual scaffolds ($\Psi$). These trajectories often end in 
“loose boundaries’’ reflecting unbound alternatives or partial information.
2) (\emph{Sleep as closure}) During sleep, hippocampal replay (ripples nested 
in slow oscillations) and REM phase coordination act as coincidence detectors 
that align these fragments. Misaligned trajectories cancel, while aligned 
trajectories are reinforced, yielding closed cycles $\gamma$ with 
$\partial \gamma = 0$.
3) (\emph{Persistence}) By $\partial^2=0$, boundaries of boundaries cannot 
survive. Thus, each daily cycle of wake and sleep eliminates trivial chains 
(forgetting) while stabilizing nontrivial homology classes $[\gamma] \in H_1(\mathcal{Z})$ 
(consolidated memories).
\end{corollary}

% ---------- TikZ FIGURE ----------
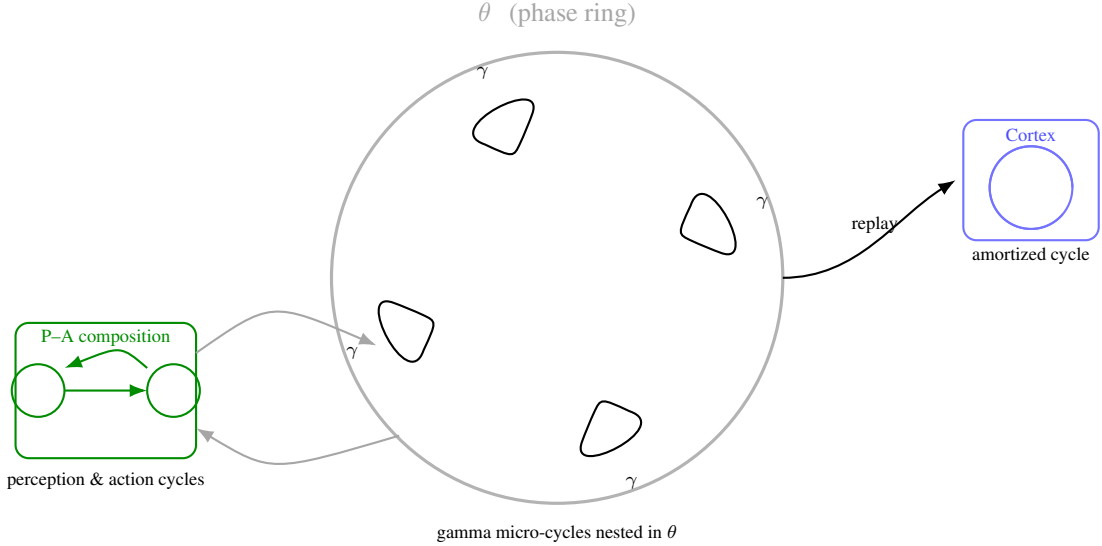
\begin{figure}[t]
\centering
\begin{tikzpicture}[>=Latex, thick, scale=1.0]
  % Theta ring
  \draw[gray!60, line width=1.2pt] (0,0) circle (3.0);
  \node[gray!70] at (0,3.5) {$\theta$ \, (phase ring)};

  % Gamma cycles at different theta phases
  \foreach \ang in {20,110,200,290}{
    \begin{scope}[rotate=\ang]
      \draw[rounded corners=6pt] (2.0,0.3) .. controls (2.4,0.8) and (2.6,-0.8) .. (2.0,-0.3)
                                .. controls (1.7,0.0) and (1.7,0.0) .. (2.0,0.3);
      \node[font=\scriptsize] at (2.9,0) {$\gamma$};
    \end{scope}
  }

  % Ripple arrow (replay) to cortex template
  \draw[->] (3.0,0) .. controls (4.0,0.0) and (4.5,0.8) .. (5.3,1.3) node[midway, above] {\scriptsize replay};
  % Cortex template (amortized cycle)
  \draw[rounded corners=5pt, blue!55] (5.4,0.5) rectangle (7.2,2.1);
  \node[blue!70] at (6.3,1.9) {\scriptsize Cortex};
  \draw[blue!55] (6.3,1.2) circle (0.55);
  \draw[blue!55] (6.3,1.2) +(0.55,0) arc (0:320:0.55);
  \node[font=\scriptsize] at (6.3,0.3) {amortized cycle};

  % Perception-action composition box
  \draw[rounded corners=5pt, green!55!black] (-7.2,-2.4) rectangle (-4.8,-0.6);
  \node[green!55!black] at (-6.0,-0.8) {\scriptsize P--A composition};
  \draw[green!55!black] (-6.9,-1.5) circle (0.35);
  \draw[green!55!black] (-5.1,-1.5) circle (0.35);
  \draw[->, green!55!black] (-6.55,-1.5) -- (-5.45,-1.5);
  \draw[->, green!55!black] (-5.45,-1.2) .. controls (-5.8,-0.9) .. (-6.55,-1.2);
  \node[font=\scriptsize] at (-6.0,-2.7) {perception \& action cycles};

  % Arrows from theta ring to P-A and back
  \draw[->, gray!70] (-2.1,-2.1) .. controls (-3.8,-2.6) .. (-4.8,-2.0);
  \draw[->, gray!70] (-4.8,-1.0) .. controls (-3.8,-0.3) .. (-2.4,-0.9);

  % Labels
  \node at (0,-3.4) {\scriptsize gamma micro-cycles nested in $\theta$};
\end{tikzpicture}
\caption{\textbf{PNG hierarchy.} Polychronous gamma cycles (micro-cycles) nest within $\theta$ phase (hippocampal scaffold, marked by gray), replayed by ripples to consolidate a cortical template (amortized macro-cycle, marked by blue). Perception and action cycles (P-A composition, as marked by green) compose via reentry, yielding order-invariant joint cycles.}
\label{fig:png-hierarchy}
\end{figure}
% ---------- END TIKZ ----------

\paragraph{Remark:} 
The sleep-wake alternation implements topological closure: every waking 
trajectory is either collapsed or consolidated, ensuring that only invariant 
cycles persist as memory. In this way, sleep acts as a homological filter, 
removing trivial boundaries while stabilizing nontrivial cycles. Together, 
Theorem~\ref{thm:cycle-hierarchy}, Theorem~\ref{thm:png_closure}, and 
Proposition~\ref{prop:png-hierarchy} provide a unified account (cf. 
Fig.~\ref{fig:png-hierarchy}): PNGs yield recurrent neural 1-cycles, 
coincidence detection enforces order invariance, and cross-frequency nesting 
plus replay extend these micro-cycles into a hierarchy of persistent homology 
classes.  Importantly, this hierarchy is not static but \emph{bootstrapped} across scales \cite{staresina2024coupled}: 
micro-level cycles provide local invariants that seed higher-order 
constructions, while macro-level cycles (e.g., the sleep-wake rhythm) 
integrate and reorganize them into more global invariants. Each cycle thus 
serves as both the \emph{substrate} for consolidating past structure and the 
\emph{scaffold} for generating new structure, enabling recursive alignment of 
context ($\Psi$) and content ($\Phi$). This recursive bootstrapping trick 
ensures that information flows coherently in both space and time \cite{li2025information}: local 
perceptual or motor fragments are bound into cycles, cycles are composed into 
plans, and daily consolidation folds these plans into long-term memory. The 
result is a multiscale information-topological framework in which cognition 
emerges from the iterative closure and recombination of cycles across nested 
timescales.

\section{Cycle is All You Need: Cognition-from-Cycle}
\label{sec:4}

At every scale of cognition, cycles serve as the primitive units of structure 
and meaning, but their deeper role is revealed from the perspective of 
\emph{bootstrapping} \cite{bennett2023brief}. In spatial navigation, closed cycles of movement provide the first scaffold: homing trajectories and grid-cell lattice paths encode 
location not by absolute coordinates, but by recurrent traversals that 
bootstrap positional knowledge from repeated experience \cite{mcnaughton2006path}. In perception, 
cycles extend this trick: oscillatory reentrant cycles synchronize distributed 
assemblies, binding features into coherent objects by reusing the same 
closure principle across sensory modalities \cite{treisman1980feature}. In action, motor trajectories 
are stabilized by the same mechanism, closing the cycle between intention and 
execution so that control is bootstrapped from trial-and-error into robust 
routines resilient to uncertainty \cite{todorov2002optimal}. These navigational, perceptual, and motor 
cycles are not separate but interwoven, forming a multi-level hierarchy where 
each domain reuses the closure principle to stabilize the next (refer to Fig. \ref{fig:three-cycles}). This recursive 
reliance on cycles embodies the \emph{bootstrapping trick}: higher-level 
cognition emerges not from new primitives, but from reapplication of the same 
topological invariant, closure under $\partial^2=0$, across scales. 
%In line with the phylogenetic continuity hypothesis \cite{buzsaki2013memory}, this view suggests that abstract thought did not arise de novo, but was bootstrapped from the ancestral substrates of navigation, perception, and action, unified by the persistence of cycles.

\subsection{Spatial Cycles: Navigation and Homing Behavior}

Navigation is not merely a sequence of displacements, but the closure 
of movement into cycles anchored at a home location. From an evolutionary 
perspective, this cyclic principle reflects the \emph{phylogenetic continuity 
between navigation and memory} \cite{buzsaki2013memory}. Long before the emergence of abstract 
representation, organisms evolved neural mechanisms for homing: returning 
to a safe base after exploration \cite{mcnaughton2006path}. What mattered was not the precise order of local excursions, but the invariant property of cycle closure, encircling 
landmarks and returning reliably to $x_0$. 
Hippocampal and entorhinal circuits exemplify this continuity \cite{o1978hippocampus}. Place cells encode discrete positions (dots, $H_0$), while grid cells impose a lattice 
that supports cycles (cycles, $H_1$). Together, they transform noisy, 
order-sensitive trajectories into invariant homological classes that persist 
as spatial memory. This organization generalizes beyond navigation: 
cognition itself inherits the cycle-based principle, with perception, action, 
and planning all grounded in cycle closure. 
At the foundation, homing provides a primordial case study of the dot-cycle dichotomy: 
local positional samples ($H_0$) acquire meaning only when integrated into 
closed cycles ($H_1$). The following lemma formalizes this principle, 
showing that homing trajectories are characterized not by the sequence 
of moves but by their homology class, capturing the order-invariant 
topological essence of navigation.

\begin{lemma}[Order-Invariant Homing in Spatial Navigation]
\label{lem:order-invariant-homing}
Let $\mathcal{X}$ be a planar configuration space (workspace) and let
$\mathcal{O}=\{O_1,\dots,O_m\}$ be disjoint compact obstacles with
$\mathcal{F} \coloneqq \mathcal{X}\setminus \bigcup_{j=1}^m O_j$ path-connected.
Fix a basepoint $x_0\in\mathcal{F}$ (``home''). 
Consider a finite library of local moves $\mathcal{M}=\{\mu_1,\dots,\mu_K\}$,
each a rectifiable path inside $\mathcal{F}$ with well-defined winding numbers
around each obstacle. For any finite sequence $w=\mu_{i_1}\cdots \mu_{i_T}$
that yields a homing trajectory $\gamma_w$ (i.e., a cycle with
$\gamma_w(0)=\gamma_w(1)=x_0$), the following holds:
1) (\emph{Order invariance in homology}) The first homology class
$[\gamma_w]\in H_1(\mathcal{F};\mathbb{Z})$ depends only on the multiset
of traversed moves (equivalently, on the net winding vector around $\{O_j\}_{j=1}^m$),
not on their execution order. In particular, any permutation of the same local moves
that remains a cycle at $x_0$ represents the same class in $H_1$.
2) (\emph{Goal invariance}) If two homing paths $\gamma,\eta$ have identical net
winding vectors about each obstacle, then $[\gamma]=[\eta]$ in $H_1(\mathcal{F})$;
thus the \emph{navigational goal} ``return to $x_0$ with given obstacle-circumnavigations''
is a cycle invariant that abstracts away the ordering of local steps.
\end{lemma}

\begin{proof}[Proof sketch]
cycles based at $x_0$ form $\pi_1(\mathcal{F},x_0)$. The Hurewicz homomorphism
$h:\pi_1(\mathcal{F},x_0)\to H_1(\mathcal{F};\mathbb{Z})$ abelianizes concatenation,
so commutators vanish in $H_1$. Consequently, concatenating cycle segments in different
orders yields the same $1$-cycle class provided the path remains closed. 
For planar $\mathcal{F}$ with punctures (obstacles), $H_1(\mathcal{F};\mathbb{Z})\cong \mathbb{Z}^m$,
with coordinates given by integer windings around each $O_j$. The image $h([\gamma_w])$ is precisely
the net winding vector obtained by summing the contributions of the local moves; summation is commutative,
hence independent of order. If two homing cycles share the same winding vector, they represent the same class
in $H_1(\mathcal{F})$, establishing both claims.
\end{proof}

\begin{example}[Spatial Cycles in Navigation and Homing]
\label{ex:spatial-cycles}
The lemma manifests in animal navigation. 
Consider an animal returning to a ``home'' location $x_0$ after exploratory 
movement in its environment $\mathcal{E} \subset \mathbb{R}^2$ \cite{muller1996hippocampus}. 
The trajectory $\gamma:[0,T]\to\mathcal{E}$ with $\gamma(0)=\gamma(T)=x_0$ 
is a closed cycle representing a $1$-cycle in $H_1(\mathcal{E})$.
1) \emph{Path integration.} In the hippocampal-entorhinal circuit \cite{OKeefeRecce1993}, 
    grid cells in MEC implement a quasi-periodic basis for location, while hippocampal 
    place cells anchor landmarks. Neural activity encodes the double integral of 
    velocity input (distance $\times$ direction), forming a dead-reckoning estimate. 
    Such integrals accumulate error unless trajectories are closed.
2) \emph{Cycle closure as error correction.} When the animal revisits 
    $x_0$, boundary conditions enforce $\gamma(0)=\gamma(T)$, canceling accumulated 
    drift. Algebraically, $\partial \gamma = 0$, so $\gamma\in Z_1(\mathcal{E})$. 
    Homing cycles thus act as calibration cycles stabilizing spatial memory 
    \cite{mcnaughton1991dead}.
3) \emph{Biological invariants.} Empirical evidence shows that hippocampal 
    place fields and entorhinal grid fields form attractor-like cycles that persist 
    across trials \cite{o1978hippocampus,hafting2005microstructure}. These cycles 
    are order-invariant: what matters is cycle closure, not the micro-sequence of turns.
\emph{Summary.} Navigation illustrates how memory emerges from recurrent traversal: 
dead-reckoning provides a provisional scaffold, while homing cycles enforce closure 
and stabilize spatial invariants \cite{mcnaughton2006path}. This mechanism grounds the phylogenetic continuity hypothesis: structured memory in higher cognition inherits its organizational principle from ancient navigation cycles.
\end{example}

\begin{corollary}[Grid-Place Homology Encodes Homing Invariance]
\label{cor:grid-place-homing}
Let $\mathcal{E}$ denote the external environment with landmarks, and 
$\mathcal{Z}$ the neural state space spanned by place and grid cell activity.
Suppose:
1) Place cells encode local positions $x \in \mathcal{E}$ as 
    ``dots,'' corresponding to $0$-cycles in $H_0(\mathcal{Z})$.
2) Grid cells impose a periodic lattice structure that embeds 
    trajectories into closed cycles on the torus $S^1 \times S^1$, corresponding 
    to $1$-cycles in $H_1(\mathcal{Z})$.
3) Homing requires a return to a reference point $x_0 \in \mathcal{E}$, 
    producing a cycle $\gamma$ in $\mathcal{E}$ whose neural counterpart 
    $\gamma^\ast$ is represented in $\mathcal{Z}$.
Then:
1) (\emph{Order invariance}) Any path in $\mathcal{E}$ that returns to 
    $x_0$ induces the same homology class $[\gamma^\ast] \in H_1(\mathcal{Z})$, 
    independent of the order of intermediate moves.
2) (\emph{Dot-cycle unification}) Place-cell activations (dots, $H_0$) 
    are bound by grid-cell cycles (cycles, $H_1$), ensuring that discrete 
    positional samples are integrated into a coherent spatial invariant.
3) (\emph{Homing memory}) The grid--place system instantiates homing 
    as topological closure: once $\partial \gamma^\ast = 0$, the trajectory 
    persists as a reusable neural cycle, invariant to perturbations or 
    path-specific details.
Thus, the grid-place cell code provides a neural realization of homing 
invariance, collapsing order-sensitive trajectories into persistent cycle 
classes that unify local exploration into global memory.
\end{corollary}

\subsection{Perceptual Cycles: Feature Binding}

Perception operates not in a single glance but as an iterative cycle of 
saccades, where successive glimpses contribute partial evidence about object 
identity and location. The challenge of feature binding lies in integrating 
these fragments into a coherent whole despite changes in fixation order and 
gaze trajectory. Much as homing paths in navigation collapse to cycle 
invariants, perceptual binding can be understood as cycle accumulation of 
feature votes: the generalized Hough transform (GHT) \cite{ballard1981generalizing} pools across glimpses 
in a way that renders both the order of saccades and their specific sequence 
irrelevant once the cycle is closed. This mechanism parallels navigation, 
where coordinate transformations from egocentric movements to allocentric 
maps yield stable spatial invariants; similarly, perception transforms 
egocentric saccadic fragments into allocentric object representations through 
cycle closure. Beyond geometric voting, the \emph{delta-homology analogy} \cite{li2025memory}
suggests a broader interpretation of GHT: memory retrieval is indexed not by 
parametric templates, but by persistent homology classes that encode the 
invariant relational structure of features. In this way, saccade-invariant 
object recognition emerges as a structure-aware retrieval process, where 
cycles both accumulate local evidence and anchor it to topological invariants 
in feature space.

\begin{proposition}[GHT-Induced Saccade and Order Invariance via Cycle Accumulation]
\label{prop:ght-saccade}
Let $G\!\le\!SE(2)$ denote the (planar) saccade group acting on retinal coordinates, and let $\mathcal{P}$ be an object
parameter space (e.g., translation--rotation--scale). Let $F$ be a local feature extractor (e.g., edges with orientation) that is
$G$-equivariant: $F(g\!\cdot\!I)=g\!\cdot\!F(I)$. Let $H$ be a generalized Hough transform that maps a multiset of features
$\mathcal{S}\subset F(I)$ to an accumulator $A:\mathcal{P}\to\mathbb{R}_{\ge 0}$ by voting with a model table $\psi$
and a kernel $\kappa$:
$A(p)\;=\!\!\sum_{f\in \mathcal{S}}\; \kappa\!\big(p-\psi(f)\big)$.
Assume per-fixation gaze state $g_t\in G$ is known (efference copy), and define re-registered votes by applying the inverse
gaze transform (or, equivalently, accumulate in the quotient $\mathcal{P}/G$).
Consider a sequence of saccadic glimpses $I_t$, $t=1,\dots,T$, with feature sets $\mathcal{S}_t=F(g_t\!\cdot\!I)$.
Define the cycle-accumulated Hough map
$\bar{A}(p)\;=\;\sum_{t=1}^{T}\;\sum_{f\in \mathcal{S}_t}\; \kappa\!\big(p-\psi(g_t^{-1}\!\cdot\! f)\big)$.
Then the following hold:
1) \textbf{Order invariance.} For any permutation $\pi$ of $\{1,\ldots,T\}$, the accumulator is unchanged:
$\bar{A}$ computed on $(\mathcal{S}_{\pi(1)},\ldots,\mathcal{S}_{\pi(T)})$ equals the original $\bar{A}$. Hence
$\arg\max_{p}\bar{A}(p)$ is invariant to the order of glimpses.
2) \textbf{Saccade invariance.} If the saccade path closes (i.e., $\prod_{t=1}^T g_t = e$) \emph{or} we re-register by $g_t^{-1}$
at each step (accumulating in $\mathcal{P}/G$), then $\bar{A}$, and thus $\arg\max_{p}\bar{A}(p)$ in object coordinates, is invariant
to the particular saccade path. Invariance holds under any rearrangement of the same multiset of observed features.
3) \textbf{Homological persistence.} View $\{\bar{A}_{\tau}\}_{\tau}$ as a filtration by threshold (superlevel sets of $\bar{A}$)
on $\mathcal{P}$. The dominant Hough peak induces a persistent $0$-cycle; when indexed over saccade time $t$, the trajectory of that
peak traces a $1$-cycle in $\mathcal{P}\times S^1$ (time on a circle), which remains stable under small perturbations of features and gaze.
\end{proposition}

\begin{proof}[Proof sketch]
(1) Order invariance follows from the commutativity of summation in $\bar{A}$; the GHT accumulator depends on the multiset of votes,
not their order (abelianization mirrors $H_1$ collapsing commutators). 
(2) By $G$-equivariance of $F$ and re-registration by $g_t^{-1}$, votes are expressed in object-centric coordinates; a closed saccade path
(or explicit quotienting by $G$) removes gaze as a nuisance, yielding saccade-invariant peaks.
(3) As $\tau$ varies, superlevel sets of $\bar{A}$ produce a persistence module; the principal component (object hypothesis) corresponds to a
persistent $0$-cycle. Across cyclic saccade time, its track is a closed cycle in $\mathcal{P}\times S^1$, stable by standard persistence
stability theorems. 
\end{proof}

The so-called ``binding problem’’ in perception \cite{treisman1980feature} is ill-posed because local 
features, edges, orientations, or colors, are ambiguous when considered in 
isolation. Without a global organizing principle, there is no unique way to 
decide which features belong together as a coherent object, particularly when 
fixations reorder the sequence of inputs or when multiple objects share 
similar features. Topological closure via cycle accumulation provides an alternative 
solution: by pooling votes across saccades and re-registering them into a 
common object-centered frame, the system enforces a coincidence detection 
principle \cite{konig1996integrator}. Misaligned features cancel as open chains, while consistent 
co-occurrences reinforce one another, closing into stable cycles. In this way, 
feature binding is not decided by arbitrary local pairings, but by the global 
constraint $\partial^2 = 0$: only those feature constellations that survive 
as closed cycles across space and time persist as object hypotheses. In this way, 
cycle accumulation transforms the underdetermined binding problem \cite{di2012feature} into a well-posed one, where invariance and persistence emerge from the algebraic 
structure of closure.

Classical approaches to feature binding often rely on Bayesian inference, 
where priors and likelihoods are combined to estimate the most probable 
configuration of features. While powerful, this framework does not by itself 
guarantee stability: competing hypotheses can oscillate, and order-dependent 
evidence integration may bias the outcome. In contrast, the cycle-closure 
approach enforces an algebraic constraint: features that do not fit into a 
closed cycle are systematically canceled, while those that align coherently 
are stabilized into homology classes. This difference is crucial: Bayesian 
models may assign probability mass to multiple, incompatible bindings, but 
the topological condition $\partial^2=0$ ensures that only closed, 
order-invariant cycles survive. From this perspective, coincidence detection 
acts not as a probabilistic weight but as a boundary operator, pruning 
spurious bindings and leaving behind only those feature constellations that 
constitute persistent object representations.

\subsection{Action Cycles: Trajectory Stabilization}

Perception and action form dual components of the cognitive cycle, operating in 
opposite informational directions \cite{poggio2004generalization}: perception maps environmental input into latent representations, while action projects latent plans back onto the environment. Closure in this perception-action cycle \cite{fuster2004upper} requires not only that perceptual fragments be bound into coherent objects, but also that motor 
commands converge onto stable trajectories that achieve intended goals. 
Without action, the cycle remains open, leaving perception unanchored to 
outcomes; without perception, action lacks contextual guidance. Importantly, 
trajectory stabilization in motor control depends on the same principle of 
order invariance that governs perceptual binding \cite{todorov2002optimal}: just as object recognition remains robust to the order of saccades, successful actions must remain 
robust to the sequencing of micro-movements \cite{shmuelof2012motor}. By collapsing order-dependent variations into invariant cycles, the system ensures that motor execution 
returns reliably to its intended target, even under perturbations \cite{santuz2020lower}. Therefore, action cycles provide the complementary half of closure, transforming latent 
plans into embodied invariants that complete the perception-action cycle.
Formally, we have

\begin{proposition}[Order Invariance in Motor Control via Cycle Formation]
\label{prop:motor-order}
Let $\mathcal{Z}$ be a latent action manifold with base state $x_0$ (rest or home posture).
Let $\mathcal{M}=\{m_1,\dots,m_k\}$ be a finite set of motor primitives, each represented as a path
$\mu_i:[0,1]\to\mathcal{Z}$ with $\mu_i(0),\mu_i(1)$ lying in neighborhoods of feasible states.
Suppose a composite action $w=m_{i_1}\cdots m_{i_T}$ produces a closed trajectory $\gamma_w$ with
$\gamma_w(0)=\gamma_w(1)=x_0$.
Then we have:
1) The homology class $[\gamma_w]\in H_1(\mathcal{Z})$ depends only on the net composition of primitives,
not on their execution order, provided the trajectory remains closed.
2) Goal-directed actions (e.g., homing or grasping) are represented as cycle classes $[\gamma_w]$, which
abstract away the precise ordering of intermediate motor commands.
3) Such classes are persistent under perturbations in execution, reflecting the robustness of motor memory.
\end{proposition}

\begin{proof}[Proof sketch]
Motor primitives concatenate as cycles in $\pi_1(\mathcal{Z},x_0)$. Passing to $H_1$, commutators vanish,
so the resulting class depends only on the cumulative 1-chain (the ``net cycle''), not the order of moves.
Thus, all permutations of a fixed set of primitives realizing the same closed trajectory define the same
homology class. Persistence follows from homology stability under perturbations.
\end{proof}

Proposition~\ref{prop:motor-order} establishes that motor execution inherits 
the same order-invariance property observed in perception: once a closed cycle 
is formed, the specific sequencing of micro-primitives becomes irrelevant, and 
goal-directed actions are encoded as persistent cycle classes. However, cognition 
does not rely on isolated perceptual or motor cycles alone. The true power of 
cycle formation emerges when these domains are coupled and recursively 
composed: perceptual cycles bind sensory fragments into coherent objects, 
motor cycles stabilize trajectories toward goals, and their integration in the 
perception-action cycle produces cross-modal invariants \cite{li2023internal}. Iterating this process 
yields a hierarchy of cycles, from micro-level trajectories to higher-order 
concepts, each stabilized by the closure principle $\partial^2=0$. The following 
theorem formalizes this hierarchical organization of cycles, showing how 
bootstrapping implements it biologically across scales.

\begin{theorem}[Bootstrapping Yields a Hierarchy of Invariant Cycles]
\label{thm:cycle-hierarchy}
Let $\mathcal{Z}$ be a latent state space with symmetry group $G$ acting on it.
Suppose inference dynamics are initially $G$-symmetric but are perturbed by
contextual selection $\Psi$, breaking the symmetry. Then we have:
1) (\emph{Local cycles}) At the micro level, each broken symmetry produces
a closed trajectory $\gamma_1\subset \mathcal{Z}$ such that $\partial \gamma_1=0$.
These cycles stabilize perceptual or motor invariants that are robust to the
order of local operations (e.g.\ saccade order in vision, primitive order in motor control).
2) (\emph{Perception-action cycles}) Coupling perceptual cycles
$[\gamma^P]\in H_1(\mathcal{Z}_P)$ with motor cycles
$[\gamma^A]\in H_1(\mathcal{Z}_A)$ yields joint cycles
$[\Gamma]\in H_1(\mathcal{Z}_P\times \mathcal{Z}_A)$.
These cycles are invariant to the ordering of both perceptual and motor steps,
including cross-modal interleavings, provided the cycle closes with respect to $\Psi$.
Thus, goal invariance emerges across modalities.
3) (\emph{Hierarchical composition}) Iterated cycle formation produces
higher-order invariants: compositions of cycles generate nontrivial homology
classes in $H_k(\mathcal{Z})$ for $k>1$, representing abstract plans and concepts
that persist across contexts. These larger cycles embody \emph{order invariance
at multiple scales}, unifying local actions into global structures.
4) (\emph{Persistence}) By $\partial^2=0$, all such cycles are closed.
Trivial cycles collapse (short-lived scaffolds), while nontrivial cycles define
persistent homology classes robust to perturbations of $\Psi$ and $f$.
Hence, memory emerges as the stable residue of broken symmetry, organized into
a nested hierarchy of invariant cycles.
\end{theorem}

\begin{proof}[Proof sketch]
(1) Follows from Lemma~\ref{lem:symm-invariant}: broken symmetry and feedback induce local cycles.
(2) Uses Theorem~\ref{thm:PA-cycle}: joint cycles in $\mathcal{Z}_P\times \mathcal{Z}_A$ are
order-invariant across perception and action. (3) Higher-order cycles follow from the
Künneth theorem for homology, which describes how tensor products of cycle classes yield
higher-dimensional invariants. (4) Stability is guaranteed by the homological identity
$\partial^2=0$, ensuring persistence of nontrivial cycles. 
\end{proof}

Theorem~\ref{thm:cycle-hierarchy} establishes that cognition is organized 
through a hierarchy of cycles, beginning with local invariants and extending 
to higher-order structures that persist across contexts. Note that perception and 
action do not operate as independent hierarchies: their true significance 
emerges when they are coupled into a single recurrent cycle. Perception 
supplies invariant features that stabilize sensory interpretation, while 
action closes trajectories that stabilize motor execution. The feedback 
between these domains ensures that neither side remains open-ended \cite{fuster2004upper}: 
perceptual cycles must be anchored to outcomes, and motor cycles must be 
guided by contextual scaffolds. This mutual dependency is captured by 
CCUP \cite{li2025CCUP}, which aligns context ($\Psi$) with content ($\Phi$) across modalities. 
The following theorem formalizes this principle by showing how the 
perception–action cycle itself constitutes a higher-order invariant: 
a closed cycle in the joint space $\mathcal{Z}_P\times\mathcal{Z}_A$, robust 
to the ordering and interleaving of perceptual and motor steps.

\begin{theorem}[Perception-Action Cycles as Higher-Order Order Invariants]
\label{thm:PA-cycle}
Let $\mathcal{Z}_P$ be a perceptual latent space with local features $\Phi$ and
$\mathcal{Z}_A$ be an action latent space with motor primitives $\mu$.
Suppose:
1) Perceptual glimpses $\gamma^P$ close to cycles $[\gamma^P]\in H_1(\mathcal{Z}_P)$
that are invariant to the order of glimpses (Proposition~\ref{prop:ght-saccade}).
2) Motor trajectories $\gamma^A$ close to cycles $[\gamma^A]\in H_1(\mathcal{Z}_A)$
that are invariant to the order of primitives (Proposition~\ref{prop:motor-order}).
3) Perception and action are coupled by a feedback operator
$\mathcal{F}:\mathcal{Z}_P\times\mathcal{Z}_A\to\mathcal{Z}_P\times\mathcal{Z}_A$
that enforces reciprocal updating (CCUP).
Then the joint trajectory
$\Gamma=(\gamma^P,\gamma^A)$
defines a cycle $[\Gamma]\in H_1(\mathcal{Z}_P\times \mathcal{Z}_A)$
that is invariant not only to the order of perceptual glimpses and the order of motor primitives separately,
but also to any interleaving of perception and action steps, provided the joint cycle closes.
This yields a \emph{perception–action cycle} as a higher-order cycle representing goal invariance across modalities.
\end{theorem}

\begin{proof}[Proof sketch]
Closure in $\mathcal{Z}_P$ and $\mathcal{Z}_A$ separately ensures $[\gamma^P]\in H_1(\mathcal{Z}_P)$
and $[\gamma^A]\in H_1(\mathcal{Z}_A)$ are well defined and order-invariant.
The product space $\mathcal{Z}_P\times \mathcal{Z}_A$ has homology given by
$H_1(\mathcal{Z}_P)\oplus H_1(\mathcal{Z}_A)$.
Interleavings of perception and action correspond to different concatenations in $\pi_1$, but
under the Hurewicz map to $H_1$ these permutations commute.
Thus $[\Gamma]$ is invariant to any reordering of perceptual or motor steps, including cross-modal orderings,
as long as the cycle closes jointly.
\end{proof}

\noindent Theorem~\ref{thm:PA-cycle} establishes that perception and action, 
when coupled through reciprocal feedback, form joint cycles that remain 
invariant under arbitrary interleavings of glimpses and motor primitives. 
This result highlights that the essence of cognition is not merely in the 
separate closure of perceptual or motor loops, but in their synchronized 
integration into higher-order invariants spanning modalities \cite{swenson1991thermodynamic}. To formalize 
how such integration is achieved, one requires a framework that treats 
local perceptual fragments and global motor goals within the same algebraic 
apparatus \cite{ayzenberg2025sheaf}. Sheaf and co-sheaf theory provide exactly this: sheaves describe 
how local features glue into coherent global percepts, while co-sheaves 
describe how global intentions decompose into locally executable primitives. 
These dual constructions allow us to generalize the topological closure 
principle ($\partial^2=0$) from joint cycles to the binding of perception 
and planning, motivating the following subsection.

% Preamble:
% \usepackage{tikz}
% \usetikzlibrary{arrows.meta,positioning,calc,fit,backgrounds,decorations.pathmorphing}

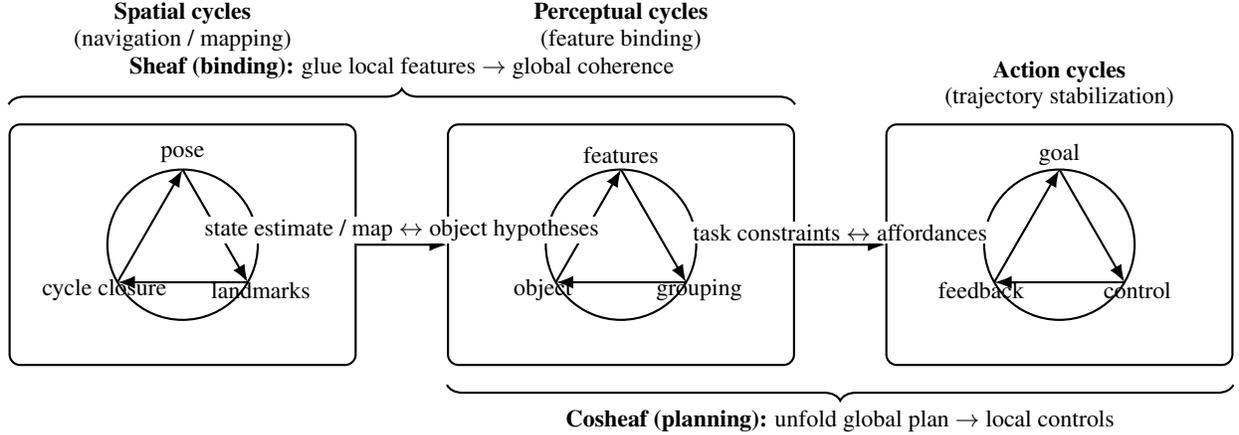
\begin{figure}[t]
\centering
\begin{tikzpicture}[
  font=\small,
  panel/.style={rounded corners, draw, thick, inner sep=6pt, minimum width=4.6cm, minimum height=3.2cm},
  cyc/.style={draw, thick},
  arrow/.style={-Latex, thick},
  lab/.style={inner sep=1pt, fill=white, anchor=south},
  every node/.style={align=center}
]

% Panels (left to right)
\node[panel] (spatial) {~};
\node[panel, right=1.2cm of spatial] (percept) {~};
\node[panel, right=1.2cm of percept] (action) {~};

% Titles
\node[above=24pt of spatial.north] {\textbf{Spatial cycles}\\(navigation / mapping)};
\node[above=24pt of percept.north] {\textbf{Perceptual cycles}\\(feature binding)};
\node[above=2pt of action.north] {\textbf{Action cycles}\\(trajectory stabilization)};

% Cycles inside panels (simple 3-node cycles)
% Spatial cycle
\path (spatial.center) node (sC) {};
\draw[cyc] ($(sC)+(0,0)$) circle [radius=1.0cm];
\foreach \ang/\lbl in {90/{pose},-30/{landmarks},210/{cycle~closure}}{
  \node at ($(sC)+({1.2*cos(\ang)},{1.2*sin(\ang)})$) {\lbl};
}
\foreach \a/\b in {90/-30,-30/210,210/90}{
  \draw[arrow] ($(sC)+({1.0*cos(\a)},{1.0*sin(\a)})$) -- ($(sC)+({1.0*cos(\b)},{1.0*sin(\b)})$);
}

% Perceptual cycle
\path (percept.center) node (pC) {};
\draw[cyc] ($(pC)+(0,0)$) circle [radius=1.0cm];
\foreach \ang/\lbl in {90/{features},-30/{grouping},210/{object}}{
  \node at ($(pC)+({1.2*cos(\ang)},{1.2*sin(\ang)})$) {\lbl};
}
\foreach \a/\b in {90/-30,-30/210,210/90}{
  \draw[arrow] ($(pC)+({1.0*cos(\a)},{1.0*sin(\a)})$) -- ($(pC)+({1.0*cos(\b)},{1.0*sin(\b)})$);
}

% Action cycle
\path (action.center) node (aC) {};
\draw[cyc] ($(aC)+(0,0)$) circle [radius=1.0cm];
\foreach \ang/\lbl in {90/{goal},-30/{control},210/{feedback}}{
  \node at ($(aC)+({1.2*cos(\ang)},{1.2*sin(\ang)})$) {\lbl};
}
\foreach \a/\b in {90/-30,-30/210,210/90}{
  \draw[arrow] ($(aC)+({1.0*cos(\a)},{1.0*sin(\a)})$) -- ($(aC)+({1.0*cos(\b)},{1.0*sin(\b)})$);
}

% Cross-layer couplings
\draw[arrow] (spatial.east) -- node[above,lab]{state estimate / map $\leftrightarrow$ object hypotheses} (percept.west);
\draw[arrow] (percept.east) -- node[above,lab]{task constraints $\leftrightarrow$ affordances} (action.west);

% Sheaf / Cosheaf annotations (top & bottom braces)
\draw[decorate, decoration={brace, amplitude=6pt}, thick]
  ($(spatial.north west)+(0,0.25)$) -- node[above=6pt]{\textbf{Sheaf (binding):} glue local features $\to$ global coherence}
  ($(percept.north east)+(0,0.25)$);

\draw[decorate, decoration={brace, mirror, amplitude=6pt}, thick]
  ($(percept.south west)+(-0, -0.25)$) -- node[below=6pt]{\textbf{Cosheaf (planning):} unfold global plan $\to$ local controls}
  ($(action.south east)+(0, -0.25)$);

\end{tikzpicture}
\caption{Cycles as the organizing principle: spatial (navigation), perceptual (binding), and action (control) cycles are coupled. A \emph{sheaf} perspective formalizes binding (local-to-global), while a \emph{cosheaf} perspective formalizes planning (global-to-local).}
\label{fig:three-cycles}
\end{figure}

\section{More is Different: from High-order Invariance to Global Consciousness}
\label{sec:5}

\noindent Earlier, we showed that homing trajectories yield \emph{order-invariant}
1-cycles: once a path closes, its homology class depends only on the multiset
of moves, not their literal order. In realistic cognition, however, many such
loops co-occur and interact across modalities (perception, action, planning),
and what matters is not merely the invariance of each loop in isolation but the
\emph{invariance of their joint behavior} under interleaving and recombination.
This is the sense in which, echoing Anderson’s dictum that ``more is different,'' \cite{anderson1972more}
closure must operate at higher orders: invariants of invariants become the
relevant objects. To make this precise, we introduce a definition that lifts
order invariance from single loops to \emph{systems of loops} whose closures
persist jointly across contexts and modalities.

\begin{definition}[High-order Invariance]
\label{def:high-order-invariance}
Let $\{\gamma_i\}_{i=1}^m$ be cycles arising from perceptual, motor, or 
cognitive subsystems, each with homology class $[\gamma_i]\in H_{k_i}(\mathcal{Z})$.
A \emph{high-order invariant} is a class $[\Gamma]\in H_K(\mathcal{Z})$, 
with $K \geq \max_i k_i$, such that:
1) Each $[\gamma_i]$ is preserved under local order permutations 
(order invariance), and
2) Their joint closure $\Gamma = \cup_i \gamma_i$ survives under 
interleaving or recombination, i.e. $\partial^2 \Gamma = 0$.
\end{definition}

\noindent Definition~\ref{def:high-order-invariance} characterizes high-order
invariance as the persistence of joint cycles across modalities and
interleavings. While this algebraic formulation captures the essence of
closure in abstract terms, it leaves open the question of how such invariants
are realized in concrete cognitive systems. In practice, perception and action
operate at different granularities: local perceptual glimpses must be assembled
into coherent wholes, while global motor intentions must be decomposed into
fine-grained actions \cite{baltieri2019nonmodular}. To connect the abstract definition with these dual
processes, we turn to the language of sheaf and co-sheaf theory \cite{curry2014sheaves}, which
formalizes precisely the local-to-global and global-to-local mappings required
for binding and planning.

\subsection{High-order Invariant: Sheaf/Co-sheaf Views of Binding and Planning}

The motivation for adopting sheaf and co-sheaf perspectives arises from the 
need to formalize how local fragments of perception and global structures of 
action are integrated into coherent cycles. In the sheaf view, local feature 
assignments (edges, colors, glimpses) must be consistently ``glued’’ together 
to yield a global percept, capturing the essence of the binding problem \cite{treisman1980feature}. In 
the co-sheaf view, global intentions or plans (e.g.\ motor goals, task-level 
strategies) must be decomposed into locally executable primitives, capturing 
the essence of planning \cite{todorov2002optimal}. These dual perspectives resonate directly with the 
dot–cycle dichotomy: trivial fragments that collapse to points ($H_0$) represent 
local data without global consistency, while nontrivial cycles ($H_1$ and higher) 
encode the persistence of closed, order-invariant structures that survive 
integration. Therefore, the sheaf formalism provides a mathematical account of 
binding from local-to-global, while co-sheaves capture the reverse process of 
global-to-local decomposition, as shown in Fig. \ref{fig:three-cycles}. Together, they extend the topological view of 
memory and action, grounding both perception and planning in the algebraic 
principle of topological closure ($\partial^2 = 0$) \cite{hatcher2002algebraic}. We start with the following lemma on binding as sheaf gluing.

\begin{lemma}[Binding as Sheaf Gluing]
\label{lem:sheaf-binding}
Let $\mathcal{X}$ be a latent context space with an open cover $\{U_i\}_{i\in I}$.
A \emph{sheaf} $\mathcal{F}$ on $\mathcal{X}$ assigns to each $U_i$ a set of local contents 
$\mathcal{F}(U_i)$ and to each inclusion $U_i \subseteq U_j$ a restriction map 
$\rho_{ij}: \mathcal{F}(U_j) \to \mathcal{F}(U_i)$.
Suppose local sections $s_i \in \mathcal{F}(U_i)$ agree on all overlaps:
$\rho_{ij}(s_j) = \rho_{ji}(s_i), \quad \forall i,j \in I$.
By the sheaf gluing axiom, there exists a unique global section 
$s \in \mathcal{F}(\bigcup_i U_i)$ restricting to each $s_i$.
1) This formalizes \emph{binding}: distributed local features (contents) can be unified into a coherent global percept 
if and only if they agree on overlaps (contexts).
2) Obstructions to global sections correspond to homology classes in $H^k(\mathcal{X};\mathcal{F})$, 
capturing unstable bindings.
\end{lemma}

Lemma~\ref{lem:sheaf-binding} shows how perception can be formalized as a 
sheaf-gluing problem: local fragments of information become meaningful only 
when they agree on overlaps, yielding a unique global section that unifies 
distributed features into a coherent percept. Note that cognition does not end 
with perception; the system must also project global intentions back into 
concrete actions. This requires the dual construction: whereas sheaves move 
from local-to-global by enforcing consistency, cosheaves move from 
global-to-local by extending higher-level structures into executable 
primitives \cite{curry2014sheaves}. In this sense, perception and planning are categorical duals: 
binding corresponds to the existence of global sections, while planning 
corresponds to the existence of global colimits. The following proposition 
makes this duality explicit, showing how cosheaf theory provides a natural 
formalism for planning as the decomposition of global goals into consistent 
local actions.

\begin{proposition}[Planning as Cosheaf Colimit]
\label{prop:cosheaf-planning}
Let $\mathcal{X}$ be a latent action manifold with an open cover $\{U_i\}_{i\in I}$.
A \emph{cosheaf} $\mathcal{G}$ on $\mathcal{X}$ assigns to each $U_i$ a set of local motor policies 
$\mathcal{G}(U_i)$ and to each inclusion $U_i \subseteq U_j$ an extension map 
$\iota_{ij}: \mathcal{G}(U_i) \to \mathcal{G}(U_j)$.
The global planning space is then the colimit
$\varinjlim_{i\in I}\, \mathcal{G}(U_i)$,
which coherently extends local primitives into larger-scale action plans.
1) This formalizes \emph{planning}: local moves or policies can be extended into a consistent global trajectory 
if their overlaps are compatible under $\iota_{ij}$.
2) Obstructions to global extension correspond to cosheaf homology classes $H_k(\mathcal{X};\mathcal{G})$, 
capturing deadlocks or infeasible plans.
\end{proposition}

Proposition~\ref{prop:cosheaf-planning} highlights the dual role of cosheaves in 
decomposing global action plans into locally executable primitives, complementing 
the sheaf-based account of perceptual binding. But perception and action do not 
operate in isolation: their integration is what closes the cycle of cognition. 
Sheaves capture how local perceptual fragments ($\Phi$) glue into coherent global 
representations, while cosheaves describe how global motor intentions ($\Psi$) 
unfold into coordinated local policies. The crucial step is their alignment: 
local perceptual sections and local motor co-sections must not only be internally 
consistent, but also mutually compatible. When this condition is satisfied, the 
pairing of sheaf and cosheaf data produces closed cycles in the nerve of the 
context cover, ensuring invariant information flow across modalities \cite{robinson2017sheaves}. The 
following corollary formalizes this synthesis, showing how a paired 
sheaf–cosheaf system naturally instantiates the perception–action cycle as a 
topological invariant.

\begin{corollary}[Perception-Action as a Paired Sheaf-Cosheaf System]
\label{cor:PA-sheaf-cosheaf}
Let $\mathcal{X}$ be a latent context space with an open cover $\mathcal{U}=\{U_i\}_{i\in I}$.
Let $\mathcal{F}$ be a sheaf of \emph{contents} on $\mathcal{X}$ and $\mathcal{G}$ a cosheaf of
\emph{policies} (motor primitives/plans) on $\mathcal{X}$. Assume a natural pairing on each open set
$\langle \cdot,\cdot\rangle_U:\ \mathcal{F}(U)\times \mathcal{G}(U)\to \mathbb{R}$
that is \emph{natural with respect to restriction/extension}: for $U\subseteq V$,
$\langle \rho_{UV}(s_V),\, g_U\rangle_U \;=\; \langle s_V,\, \iota_{UV}(g_U)\rangle_V,
\quad s_V\in \mathcal{F}(V),\ g_U\in \mathcal{G}(U)$.
Suppose (i) local sections $\{s_i\in \mathcal{F}(U_i)\}$ satisfy Čech compatibility on all overlaps
(Lemma~\ref{lem:sheaf-binding}), and (ii) local co-sections $\{g_i\in \mathcal{G}(U_i)\}$ are
mutually extendable along overlaps (Proposition~\ref{prop:cosheaf-planning}). If the pairing is
nondegenerate on all pairwise overlaps $U_{ij}=U_i\cap U_j$, then:
1) (\emph{Global Perception-Action alignment}) There exists a unique global section
$s\in \mathcal{F}(\bigcup_i U_i)$ and a canonical colimit element
$[g]\in \varinjlim_{i}\mathcal{G}(U_i)$ such that all pairing squares commute:
$\langle s|_{U_i},\, g_i\rangle_{U_i}$ is consistent across overlaps.
2) (\emph{cycle formation in the nerve}) Define the Čech $1$-cochain on the nerve
$N(\mathcal{U})$ by $\omega_{ij} \coloneqq \langle s_i, g_j\rangle_{U_{ij}} - \langle s_j, g_i\rangle_{U_{ij}}$.
Naturality implies $\delta \omega = 0$, hence $\omega$ determines a closed $1$-cycle
$[\gamma]\in H_1\!\big(N(\mathcal{U});\mathbb{R}\big)$ that is independent of the order of
refinement or interleaving of local perception/action updates.
3) (\emph{Invariant information flow}) The class $[\gamma]$ is persistent under small
perturbations of restrictions $\rho$ and extensions $\iota$ (by $\delta^2=0$ / $\partial^2=0$ in the
Čech chain complex). 
\end{corollary}

\begin{proof}[Proof sketch]
(1) Existence/uniqueness of $s$ follows from sheaf gluing; existence of $[g]$ from the cosheaf colimit.
Naturality of the pairing yields commuting squares on overlaps. (2) The definition of $\omega$ and the
commuting squares imply its Čech coboundary vanishes, so it represents a closed $1$-cycle on the nerve;
order-invariance follows from abelianization at the cohomological level. (3) Stability follows from the
functoriality of Čech (co)homology under small perturbations of the cover and maps, hence persistence of
$[\gamma]$ as the Perception-Action invariant. \qedhere
\end{proof}

\begin{example}[Eye-Hand Coordination as a Sheaf-Cosheaf System]
\label{ex:eye-hand-sheaf-cosheaf}
Consider a grasping task in which visual perception (sheaf $\mathcal{F}$) 
and hand movements (cosheaf $\mathcal{G}$) must align on a common latent 
workspace $\mathcal{X}$ (the scene of objects). Let the cover 
$\mathcal{U}=\{U_i\}$ correspond to overlapping gaze regions, each $U_i$ 
containing a partial view of the object.  
1) \emph{Sheaf of contents ($\mathcal{F}$).} Each $U_i$ encodes a 
    local percept $s_i\in\mathcal{F}(U_i)$ (e.g.\ an edge, surface, or 
    partial contour). Restriction maps $\rho_{ij}$ enforce that on overlaps 
    $U_i\cap U_j$, these local percepts agree on shared features (e.g.\ 
    the same edge corner is seen from two saccades).
2) \emph{Cosheaf of policies ($\mathcal{G}$).} Each $U_i$ also admits 
    local motor policies $g_i\in\mathcal{G}(U_i)$ (e.g.\ a partial hand 
    posture consistent with grasping the visible portion). Extension maps 
    $\iota_{ij}$ enforce that local postures can be extended when more 
    of the object is revealed.
3) \emph{Pairing.} On overlaps $U_{ij}$, the natural pairing 
    $\langle s_i, g_j\rangle_{U_{ij}}$ evaluates whether the percept of a 
    feature supports the corresponding local motor posture (e.g.\ the 
    contour seen in $s_i$ matches the finger orientation in $g_j$).
By Corollary~\ref{cor:PA-sheaf-cosheaf}:
1) (\emph{Global Perception-Action alignment}) If all local percepts $\{s_i\}$ agree 
    on overlaps and all local postures $\{g_i\}$ extend consistently, then 
    there exists a unique global percept $s$ (a complete object representation) 
    and a global motor plan $[g]$ (a consistent grasp).
2) (\emph{cycle formation}) The consistency of pairings across overlaps 
    defines a closed $1$-cycle in the Čech nerve $N(\mathcal{U})$: as gaze and 
    hand sequentially cover overlapping regions, the cycle guarantees order 
    invariance of the perception-action alignment.
3) (\emph{Invariant information flow}) Small variations in saccade 
    order or finger sequencing do not alter the global alignment, since the 
    homology class $[\gamma]\in H_1(N(\mathcal{U}))$ persists under such 
    perturbations.
\emph{Summary.} Eye-hand coordination illustrates how the sheaf (perceptual 
contents) and cosheaf (motor policies) glue together only upon cycle closure, 
ensuring that grasp succeeds invariantly of the order in which local views 
and micro-actions are integrated.
\end{example}

\paragraph{Remark:}
The paired sheaf-cosheaf structure instantiates a perception-action cycle
as an \emph{invariant carrier} of information: context (sheaf restrictions) and
content (cosheaf extensions) align dynamically only upon cycle closure, after
which the invariant can be amortized as a reusable template. In this sense, the
perception-action cycle is not merely a mechanism for online coordination but a
\emph{structural principle} of cognition: closure transforms transient
interactions into stable carriers of meaning, while amortization ensures that
these invariants can be recalled, recombined, and redeployed across contexts.  

\subsection{Consciousness as High-order Invariance of Cognition}

The preceding development established cycles as the structural carriers of 
order-invariant memory. We now extend this perspective to consciousness, 
understood as the highest-order invariance that cognition can sustain. 
Just as memory arises when transient fragments are promoted into persistent 
loops, consciousness emerges when perception and action themselves close 
into global invariants that survive across contexts \cite{li2025memory}. Sheaf-cosheaf duality 
provides the natural language for this extension.

\begin{principle}[Sheaf-Cosheaf Closure and Consciousness]
Let $\mathcal{F}$ be a sheaf encoding perceptual restrictions 
(integration of local data) and $\mathcal{G}$ a cosheaf encoding 
content extensions (differentiation into global actions). 
Conscious awareness corresponds to invariants $[\gamma]$ that lie 
in the intersection of their global sections:
$[\gamma] \;\in\; H^0(\mathcal{F}) \;\cap\; H_0(\mathcal{G})$,
so that perception and action align consistently under cycle closure 
($\partial^2=0$). 
\noindent Thus, consciousness emerges as the global fixed-point where 
sheaf-mediated integration and cosheaf-mediated differentiation 
converge into a stable loop, yielding a unified yet diverse field 
of experience.
\end{principle}

While the sheaf-cosheaf framework captures the dual roles of perceptual 
integration and action-driven differentiation, it remains to show how these 
closures yield a principled notion of invariance. This is provided by 
topological homology and probabilistic exchangeability \cite{deFinetti1937}, two complementary 
formalizations of the same structural principle: consciousness arises when 
order-sensitive fragments are absorbed into order-invariant carriers of 
meaning.

\begin{theorem}[Consciousness as Invariant Closure]
Conscious experience arises when cognition promotes transient, 
order-dependent fragments into high-order invariants. This admits two 
dual formalizations:
1) \textbf{Topological closure:} Given a chain complex 
    $(C_\bullet,\partial)$ of neural or cognitive events, closure 
    $\partial^2=0$ ensures that boundaries vanish, and only nontrivial 
    cycles $[\gamma]\in H_k(C_\bullet)$ persist. Conscious awareness 
    corresponds to the survival of such global cycles across contexts.
2) \textbf{Exchangeable inference:} Given a temporal sequence 
    $X_{1:T}$ conditionally exchangeable under a generative model, 
    De Finetti's theorem implies the existence of a latent variable 
    $\Phi^*$ such that
   $ P(X_{1:T}) = \int \prod_{t=1}^T P(X_t|\Phi^*) \, dP(\Phi^*)$.    
Conscious awareness corresponds to inference over $\Phi^*$, 
    rendering observations invariant to ordering once conditioned.
Both perspectives identify consciousness as the 
\emph{emergence of high-order invariance}: topological closure 
guarantees cycle persistence, while probabilistic exchangeability 
guarantees order-invariant inference. Emotion supplies the prior 
$P(\Phi^*)$, and cognition implements recursive closure via 
posterior updating $P(\Phi^*|X_{1:t})$.
\end{theorem}

The invariant-closure theorem clarifies the structural form of conscious 
experience, but consciousness also has a distinctive phenomenology: it is 
simultaneously unified and differentiated \cite{bayne2018axiomatic}. Closure reconciles this 
apparent opposition. When boundaries vanish, all events participating in 
a cycle become globally bound, yielding integration. Yet the space of 
homology classes allows multiple distinct cycles to coexist, yielding 
differentiation. Thus, the same algebraic principle accounts for both 
unity and richness.

\begin{proposition}[Closure Supports Integrated and Differentiated Consciousness]
Let $(C_\bullet,\partial)$ be a chain complex of cognitive events with 
$\partial^2=0$. Conscious awareness corresponds to nontrivial cycles 
$[\gamma] \in H_k(C_\bullet)$ that persist across time and context. Then:
1) \textbf{Integration:} All events that participate in a cycle 
    are globally bound through closure, since $\partial^2=0$ guarantees 
    that local inconsistencies cancel. This yields a unified invariant 
    that integrates fragments into one coherent structure.
2) \textbf{Differentiation:} Distinct cycles 
    $[\gamma_1], [\gamma_2], \dots \in H_k(C_\bullet)$ 
    provide a multiplicity of invariants, allowing diverse and 
    context-specific contents to coexist. Consciousness is thus 
    differentiated by the repertoire of independent cycles.
\noindent In this way, topological closure reconciles integration and 
differentiation: both emerge naturally when boundaries vanish and 
cycles persist as homological invariants.
\end{proposition}

Finally, these structural insights connect directly to contemporary 
theories of consciousness. Integrated Information Theory (IIT) \cite{tononi2016integrated} posits 
that consciousness requires both integration and differentiation, 
quantified by $\Phi$. Topological closure provides a natural reformulation 
of these axioms: $\Phi$ can be viewed as a persistence index of homological 
invariants, measuring the robustness and richness of the cycles that 
survive across scales. In this way, the homological account situates IIT’s 
principles within the broader framework of topological invariance. 

\begin{proposition}[Topological Closure and Integrated Information]
In Integrated Information Theory (IIT), consciousness requires both 
integration and differentiation, quantified by $\Phi$. 
From the topological closure perspective:
1) \textbf{Integration:} Closure $\partial^2=0$ ensures that 
    all local inconsistencies vanish, binding events into a global 
    invariant cycle. This captures the indivisibility of conscious 
    experience.
2) \textbf{Differentiation:} Distinct nontrivial homology classes 
    in $H_k(C_\bullet)$ provide a repertoire of independent invariants, 
    corresponding to the diversity of conscious contents.
3) \textbf{Quantification:} The integrated information measure 
    $\Phi$ can be reinterpreted as a persistence index of homological 
    invariants across scales, so that high $\Phi$ corresponds to systems 
    with globally integrated yet richly differentiated cycles.
\noindent Hence IIT’s axioms of integration and differentiation align 
naturally with the homological invariance guaranteed by topological closure.
\end{proposition}

\begin{figure*}[t]
\centering
\begin{tikzpicture}[
  >=Latex,
  scale=1.0,
  every node/.style={font=\small},
  title/.style={font=\small\bfseries, text=gray!70},
  bar/.style={line width=2.0pt, line cap=round},
  faint/.style={gray!60, bar},
  strong/.style={blue!70, bar},
  axisline/.style={gray!55, line width=0.6pt}
]

% Panels
\coordinate (L) at (-6,0);
\coordinate (R) at ( 2,0);

% ----- Left: Low-phi (fragmented, short-lived bars) -----
\node[title] at (-6,2.2) {Low $\Phi$: fragmented \& short-lived};
% axis
\draw[axisline] (-8,1.6) -- (-4,1.6);
\draw[axisline] (-8,0.0) -- (-4,0.0);
\draw[axisline] (-8,-1.6) -- (-4,-1.6);
% bars (short + many)
\draw[faint] (-7.7,1.6) -- (-7.1,1.6);
\draw[faint] (-7.3,1.6) -- (-7.0,1.6);
\draw[faint] (-6.9,1.6) -- (-6.5,1.6);
\draw[faint] (-6.5,1.6) -- (-6.3,1.6);
\draw[faint] (-6.1,1.6) -- (-5.7,1.6);
\draw[faint] (-7.8,0.0) -- (-7.4,0.0);
\draw[faint] (-7.2,0.0) -- (-6.9,0.0);
\draw[faint] (-6.8,0.0) -- (-6.2,0.0);
\draw[faint] (-6.1,0.0) -- (-5.8,0.0);
\draw[faint] (-7.6,-1.6) -- (-7.3,-1.6);
\draw[faint] (-7.1,-1.6) -- (-6.7,-1.6);
\draw[faint] (-6.6,-1.6) -- (-6.4,-1.6);
\draw[faint] (-6.3,-1.6) -- (-5.9,-1.6);

\node[gray!65, align=left] at (-6,-2.3) {\scriptsize Few/weak closures\\[-2pt]\scriptsize Low integration; many fleeting fragments};

% ----- Right: High-phi (few but long, integrated \& differentiated) -----
\node[title] at (2,2.2) {High $\Phi$: integrated \& differentiated};
% axis
\draw[axisline] (0,1.6) -- (4,1.6);
\draw[axisline] (0,0.0) -- (4,0.0);
\draw[axisline] (0,-1.6) -- (4,-1.6);
% bars (long + distinct)
\draw[strong] (0.2,1.6) -- (3.6,1.6);        % long persistent cycle A
\draw[strong] (0.5,0.0) -- (3.1,0.0);        % long persistent cycle B (distinct)
\draw[strong] (0.8,-1.6) -- (3.4,-1.6);      % long persistent cycle C
% a few short transients (still present but minor)
\draw[faint] (0.3,1.6) -- (0.7,1.6);
\draw[faint] (0.6,0.0) -- (0.9,0.0);

\node[gray!65, align=left] at (2,-2.3) {\scriptsize Strong global closure\\[-2pt]\scriptsize Few long bars = integration; multiple long bars = differentiation};

% ----- Legend -----
\begin{scope}[shift={( -1.8,-3.4 )}]
  \draw[strong] (1,0) -- (2.2,0); \node[right] at (2.3,0) {\scriptsize persistent cycle (integrated)};
  \draw[faint]  (-5.0,0) -- (-4.5,0); \node[right] at (-4.3,0) {\scriptsize transient fragment (unconsolidated)};
  %\draw[faint]  (0,-0.6) -- (0.7,-0.6); \node[right] at (1.3,-0.6) {\scriptsize transient fragment (unconsolidated)};
\end{scope}

\end{tikzpicture}
\caption{\textbf{Persistent homology view of integrated information \(\Phi\).}
Topological closure (\(\partial^2=0\)) yields bars in a persistence barcode.
\emph{Integration} corresponds to the presence of long, system-spanning bars (global cycles);
\emph{differentiation} corresponds to multiple independent long bars (distinct invariants).
Low \(\Phi\): many short, fragile bars (weak closure). High \(\Phi\): few long, robust bars
with diversity (strong closure + rich repertoire).}
\label{fig:phi-barcode}
\end{figure*}
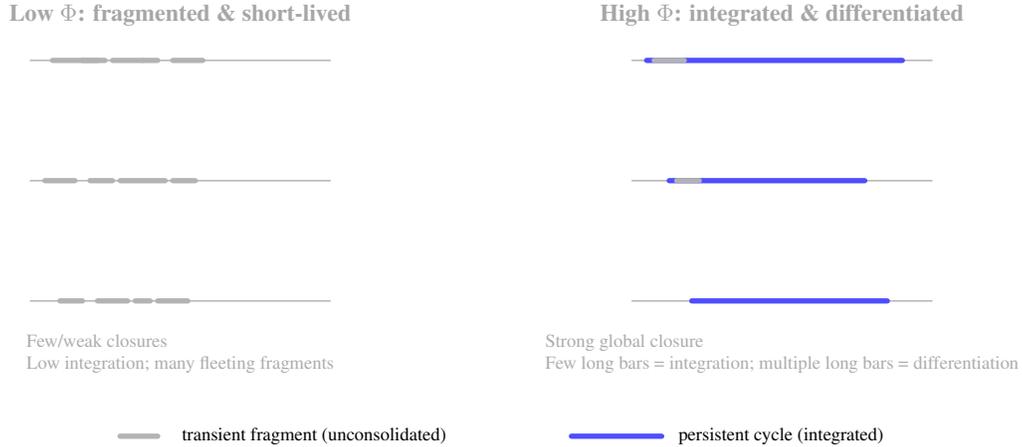

\section{Implications for Cognitive Computation and AGI}
\label{sec:6}

The framework developed above carries several implications for the design of 
artificial general intelligence (AGI) \cite{bubeck2023sparks} and the broader theory of cognitive computation \cite{kriegeskorte2018cognitive}. At its core is the \emph{dot-cycle dichotomy}: trivial cycles 
(collapsing to dots, $H_0$) represent transient scaffolds, while nontrivial 
cycles ($H_1$ and higher) embody persistent invariants. This distinction offers 
a blueprint for constructing \emph{non-Turing memory machines} whose organizing 
principle is topological closure rather than sequential symbol manipulation.

\paragraph{From Turing memory to topological memory.}
Classical Turing computation treats memory as an unstructured tape of bits, 
with read/write operations that scale poorly with dimensionality and context \cite{sipser1996introduction}. 
In contrast, the dot--cycle dichotomy suggests that memory should be organized 
as a hierarchy of homology classes: transient dots scaffold exploration, while 
closed cycles stabilize and generalize invariants. Such topological memory 
machines naturally support \emph{order invariance}, \emph{context--content 
alignment}, and \emph{amortized reuse}, properties absent in classical tape-based 
models. \emph{Order invariance} follows because once a trajectory closes, its class 
$[\gamma]\in H_1$ depends only on the multiset of contributing moves (and net 
orientations), not on their sequence; thus recognition and control rely on 
relational structure rather than literal order. \emph{Context--content alignment} 
arises from sheaf--cosheaf closure: perceptual restrictions (sheaf) and action 
extensions (cosheaf) must agree on overlaps; boundary terms cancel ($\partial^2=0$), 
so only globally consistent loops persist as memory-bearing states. 
\emph{Amortized reuse} comes from representing skills and concepts as stable cycles 
that can be replayed, composed, and parameterized without re-solving the underlying 
search; the invariant template (homology class) is retrieved and only its 
bounding cobordisms are adapted to the new task. Together, these properties turn 
memory from a brittle, order-sensitive tape into a compositional library of 
closed, reusable invariants that support robust generalization and efficient planning.

%\paragraph{Energy efficiency through amortization.}
%A key limitation of Turing-style and deep learning systems is the high energy cost of repeated inference. Memory-amortized inference (MAI) leverages cycle closure: once a cycle is formed, its invariant class can be reused without recomputing from scratch. This structural reuse drastically reduces information-theoretic redundancy, yielding a path toward \emph{energy-efficient cognition}. From a physics perspective, closure enforces $\partial^2=0$, guaranteeing that energy is not wasted on propagating boundaries of boundaries.

\paragraph{Generalization by structural invariants.}
Deep learning generalizes statistically by interpolation across training data, 
but it often fails under distributional shift. In contrast, homological memory 
generalizes structurally: cycles collapse irrelevant order-dependent noise and 
preserve only topologically invariant relations. This means that the same 
homology class can support recognition or action across novel contexts, 
providing a principled route toward robust generalization in AGI.
%\paragraph{Beyond Turing: cycle-based computation.}
A key limitation of Turing-style and deep learning systems is the high energy cost of repeated inference \cite{tkachenko2025thermodynamic}. Memory-amortized inference (MAI) leverages cycle closure: once a cycle is formed, its invariant class can be reused without recomputing from scratch \cite{li2025Beyond}. This structural reuse drastically reduces information-theoretic redundancy, yielding a path toward \emph{energy-efficient cognition}. From a physics perspective, closure enforces $\partial^2=0$, guaranteeing that energy is not wasted on propagating boundaries of boundaries.
These insights motivate a shift from \emph{symbolic 
computation} (Turing machines) and \emph{statistical computation} (deep neural 
nets) toward \emph{structural computation} \cite{izhikevich2006polychronization}: memory and inference are realized 
through cycles stabilized by algebraic invariants. In this view, cognition is 
not tape manipulation but \emph{cycle navigation} \cite{buzsaki2013memory}, with persistence of 
nontrivial cycles serving as the substrate for reasoning, planning, and 
abstraction. This cycle-based architecture opens a path toward non-Turing 
models of AGI that are both more biologically plausible and more computationally 
efficient.

\paragraph{Outlook.}
The dot-cycle dichotomy thus suggests a research program: build cognitive 
machines whose memory is structured not as a flat sequence of bits, but as a 
hierarchy of persistent cycles. Such machines promise three intertwined 
advantages: (i) energy efficiency through amortization, (ii) robust 
generalization through topological invariance, and (iii) scalability through 
hierarchical cycle composition. In this way, the algebraic principle 
$\partial^2=0$ may serve as the foundation for the next paradigm of 
cognitive computation.
\noindent While cycle-based computation offers a principled path toward robust 
and efficient cognitive machines, one must also acknowledge its intrinsic limits. 
If cognition itself generates mathematics as an abstraction of closure operations, 
then any formal attempt to capture cognition risks falling prey to the same 
self-referential barriers identified by Cantor \cite{Cantor1891} and Turing \cite{Kleene1952_MathematicalLogic}. This motivates the 
following principle.

\begin{corollary}[Diagonal Non-Closure of Cognitive Formalization]
Let cognition generate mathematics as the abstraction of its own 
closure operations. Any attempt to capture cognition exhaustively 
within mathematics encounters a diagonalization barrier analogous 
to Cantor's and Turing's arguments: the generative system cannot be 
fully enclosed by the formal system it spawned. 
Therefore, cognition admits mathematical description but not 
final reduction, since each closure produces higher-order invariants 
beyond the scope of the current formalism.
\end{corollary}

\noindent Taken together, these insights suggest a reframing of the AGI project. 
Rather than seeking to exhaustively formalize cognition within a fixed 
computational paradigm, we should aim to engineer systems whose very mode 
of operation is \emph{structural closure}: pruning inconsistencies, 
stabilizing cycles, and continually generating higher-order invariants. 
Such machines would not be bound by static rules or brittle statistical 
interpolation, but would instead embody cognition’s own grammar of 
self-organizing invariance. In this sense, the pursuit of AGI is less the 
construction of a universal Turing machine, and more the cultivation of a 
universal cycle navigator, an architecture that grows by closure, adapts by 
amortization, and transcends its own formalisms by diagonal extension \cite{Dauben1979_Cantor}.

\bibliographystyle{plain}
\bibliography{references}  %%% Uncomment this line and comment out the ``thebibliography'' section below to use the external .bib file (using bibtex) .

%%%%%%%%%%%%%%%%%%%%%%%%%%%%%%%%%%%%%%%%%%%%%%%%%%%%%%%%%%%%
\newpage
\appendix

\section{Background on Homology}

To formalize cycle invariance, we recall basic notions from algebraic topology. 
Let $X$ be a topological space. For each integer $k \geq 0$, the \emph{$k$-th homology group} 
$H_k(X)$ is defined as
\[
H_k(X) \;=\; Z_k(X) / B_k(X),
\]
where
\begin{itemize}
    \item $C_k(X)$ denotes the group of $k$-chains, i.e.\ formal sums of oriented $k$-simplices in $X$,
    \item $\partial_k : C_k(X) \to C_{k-1}(X)$ is the boundary operator,
    \item $Z_k(X) = \ker(\partial_k)$ is the group of $k$-cycles (chains with zero boundary),
    \item $B_k(X) = \operatorname{im}(\partial_{k+1})$ is the group of $k$-boundaries (boundaries of $(k+1)$-chains).
\end{itemize}
Thus, elements of $H_k(X)$ are equivalence classes of $k$-cycles modulo boundaries:
\[
[\gamma] \in H_k(X) 
\quad \Longleftrightarrow \quad \partial \gamma = 0 
\quad \text{and} \quad 
\gamma \sim \gamma + \partial \sigma.
\]

Intuitively, $H_k(X)$ captures the $k$-dimensional ``holes'' of the space $X$. 
For $k=1$, $H_1(X)$ classifies cycles up to continuous deformation, distinguishing 
nontrivial cycles from those that can be contracted to a point. In our context, 
$[\gamma] \in H_1(X)$ encodes the persistence of a recurrent cycle in state space, 
which remains invariant to local re-orderings of trajectories.

\section{Background on Filtration, Sheaves, and Cosheaves}

To formalize the dual roles of context and content in cycle formation, we recall 
basic notions from algebraic topology and category theory. These provide the 
mathematical language for scaffolding (sheaves) and planning (cosheaves).

\subsection{Filtrations}
A \emph{filtration} of a topological space $X$ is a nested sequence of subspaces
\[
\emptyset = X_0 \subseteq X_1 \subseteq X_2 \subseteq \cdots \subseteq X,
\]
often indexed by a parameter (e.g.\ time, scale, or threshold). Filtrations induce 
a corresponding sequence of homology groups
\[
H_k(X_0) \to H_k(X_1) \to H_k(X_2) \to \cdots,
\]
whose persistence identifies which cycles survive across scales. In cognition, 
filtrations represent the progressive accumulation of evidence: trivial cycles 
appear and disappear quickly, while nontrivial cycles that persist encode stable 
memory traces.

\subsection{Sheaves}
Let $\mathcal{U} = \{U_i\}_{i\in I}$ be an open cover of a topological space $X$.  
A \emph{sheaf} $\mathcal{F}$ on $X$ assigns to each open set $U \subseteq X$ 
a set (or vector space) of local data $\mathcal{F}(U)$, together with restriction 
maps
\[
\rho_{UV}: \mathcal{F}(V) \to \mathcal{F}(U), \qquad U \subseteq V,
\]
that satisfy two axioms:
\begin{enumerate}
    \item (\emph{Identity}) $\rho_{UU} = \mathrm{id}$ for all $U$,
    \item (\emph{Gluing}) If $\{s_i \in \mathcal{F}(U_i)\}$ agree on overlaps 
    $U_i \cap U_j$, then there exists a unique global section 
    $s \in \mathcal{F}(\bigcup_i U_i)$ restricting to each $s_i$.
\end{enumerate}
Intuitively, sheaves formalize the integration of local information into 
a coherent global whole. In perception, local features (edges, glimpses, 
or sensory fragments) correspond to sections $s_i$, and binding succeeds 
iff they glue into a global percept.

\subsection{Cosheaves}
Dually, a \emph{cosheaf} $\mathcal{G}$ on $X$ assigns to each open set 
$U \subseteq X$ a set (or vector space) of local data $\mathcal{G}(U)$, 
together with extension maps
\[
\iota_{UV}: \mathcal{G}(U) \to \mathcal{G}(V), \qquad U \subseteq V,
\]
that satisfy:
\begin{enumerate}
    \item (\emph{Identity}) $\iota_{UU} = \mathrm{id}$ for all $U$,
    \item (\emph{Co-gluing}) Local elements $\{g_i \in \mathcal{G}(U_i)\}$ 
    can be consistently extended into a global element 
    $g \in \varinjlim_{i}\mathcal{G}(U_i)$ if their extensions 
    agree on overlaps.
\end{enumerate}
Cosheaves thus formalize the decomposition of global structure into 
locally executable primitives. In planning, a global goal corresponds 
to a colimit element $[g]$, assembled consistently from local motor 
policies $g_i$.

\medskip

\noindent
\textbf{Intuition.} Sheaves (local-to-global) capture \emph{binding in perception}, 
while cosheaves (global-to-local) capture \emph{extension in planning}. Filtrations 
provide the temporal and hierarchical scaffolding in which both processes occur. 
Together, these concepts unify perceptual integration, motor planning, and 
multi-scale cycle formation within the same topological framework.

\section{Hopf Bifurcation and the Emergence of Cycles}
\label{sec:hopf}

One of the most fundamental mechanisms by which cycles arise in dynamical systems 
is the \emph{Hopf bifurcation}. It provides a generic route through which a 
system with a symmetric equilibrium can spontaneously generate oscillations once 
symmetry is broken. This mechanism is particularly relevant for CCUP and MAI: 
the transition from neutral symmetry to recurrent cycles mirrors the 
context-content alignment process, where a stable cycle $\gamma$ emerges 
and is subsequently amortized as memory.

\subsection{Definition}
Consider a smooth dynamical system
$\dot{x} \;=\; f(x;\mu), \quad x \in \mathbb{R}^n,\ \mu \in \mathbb{R}$,
where $\mu$ is a bifurcation parameter. Suppose the system has an equilibrium 
$x^\star(\mu)$ for all $\mu$ near $0$. A Hopf bifurcation occurs at $\mu=0$ if:

\begin{enumerate}
    \item The Jacobian $Df(x^\star(0);0)$ has a simple pair of purely 
    imaginary eigenvalues $\lambda_\pm(0)=\pm i \omega_0$, with $\omega_0 > 0$, 
    and no other eigenvalues on the imaginary axis.
    \item The real part of the eigenvalues crosses zero transversally as $\mu$ 
    varies: $\tfrac{d}{d\mu}\,\mathrm{Re}\,\lambda_\pm(0) \neq 0$.
\end{enumerate}

Then, by the \emph{Hopf bifurcation theorem}, for sufficiently small $\mu$, the system 
admits a family of periodic orbits $\gamma_\mu$ bifurcating from $x^\star(0)$. 
In the \emph{supercritical case}, these orbits are stable limit cycles.

\subsection{Biological Interpretation}
In neural systems, the equilibrium corresponds to a balanced but 
symmetry-constrained state---for instance, when excitatory and inhibitory inputs 
cancel at baseline. The critical eigenvalue pair represents a mode of 
oscillation (e.g., theta frequency) that is neutrally stable under symmetry. 
Perturbing the context variable $\Psi$ breaks this symmetry, shifting the 
eigenvalues and giving rise to a self-sustained oscillation. The resulting 
limit cycle provides a natural scaffold for embedding content variables $\Phi$, 
transforming transient dynamics into a recurrent, closed trajectory.

\subsection{Connection to CCUP and MAI}
From the perspective of CCUP, the Hopf bifurcation implements the transition 
from high-entropy symmetry (all directions equivalent, maximal uncertainty) 
to low-entropy cycles (residual invariants that persist). The closed orbit 
$\gamma_\mu$ realizes the homological identity $\partial^2=0$, since trajectories 
that fail to close cancel out, while the recurrent cycle persists as an invariant. 
MAI then \emph{amortizes} this cycle: once formed, $\gamma_\mu$ can be stored, 
recalled, and recombined, reducing the cost of future alignment between 
context ($\Psi$) and content ($\Phi$). Thus, Hopf bifurcation provides the 
dynamical substrate for cycle formation, while homology ensures its persistence.

\end{document}